\documentclass[twoside,11pt]{article}

\usepackage{blindtext}
\newtheorem{assumption}{Assumption}
\usepackage{jmlr2e}
\usepackage{amsmath}
\allowdisplaybreaks[4] % 允许自动分页，强度为4

\usepackage{lastpage}
\def\bz{{\bf z}}

\def\X{\mathcal{X}}
\def\Y{\mathcal{Y}}

\def\H{\mathcal{H}}

\begin{document}
	
	\title{Spectral Algorithms under Covariate Shift}
	
	\author{\name Jun Fan \email junfan@hkbu.edu.hk \\
		\addr Department of Mathematics\\
		Hong Kong Baptist University\\
		Kowloon, Hong Kong
		\AND
		\name Zheng-Chu Guo \email guozc@zju.edu.cn \\
		\addr School of Mathematical Sciences\\
		Zhejiang University\\
		Hangzhou 310058, China
		\AND
		\name Lei Shi \email leishi@fudan.edu.cn \\
		\addr School of Mathematical Sciences and Shanghai\\ Key Laboratory for
		Contemporary Applied Mathematics\\
		Fudan University\\
		Shanghai 200433, China}
	
	\editor{My editor}
	
	\maketitle

\begin{abstract}%   <- trailing '%' for backward compatibility of .sty file
Spectral algorithms leverage spectral regularization techniques to analyze and process data, providing a flexible framework for addressing supervised learning problems. To deepen our understanding of their performance in real-world scenarios where the distributions of training and test data may differ, we conduct a rigorous investigation into the convergence behavior of spectral algorithms under covariate shift. In this setting, the marginal distributions of the input data differ between the training and test datasets, while the conditional distribution of the output given the input remains unchanged. Within a non-parametric regression framework over a reproducing kernel Hilbert space, we analyze the convergence rates of spectral algorithms under covariate shift and show that they achieve minimax optimality when the density ratios between the training and test distributions are uniformly bounded. However, when these density ratios are unbounded, the spectral algorithms may become suboptimal. To address this issue, we propose a novel weighted spectral algorithm with normalized weights that incorporates density ratio information into the learning process. Our theoretical analysis shows that this normalized weighted approach achieves optimal capacity-independent convergence rates, but the rates will suffer from the saturation phenomenon. Furthermore,
by introducing a weight clipping technique, we demonstrate that the convergence rates of the weighted spectral algorithm with clipped weights can approach the optimal capacity-dependent convergence rates arbitrarily closely. This improvement resolves the suboptimality issue in unbounded density ratio scenarios and advances the state-of-the-art by refining existing theoretical results.
\end{abstract}


	\section{Introduction and Main Results}\label{section: introduction and main results}

In machine learning \citep{Bauer07, Caponnetto07,cucker2007learning, Lu18}, it is typically assumed that the training and test samples are drawn from the same underlying distribution. This assumption is fundamental to reliably generalize the learned patterns from training data to unseen test data. However, it is essential to recognize that this assumption does not always hold in practice. There are numerous scenarios where the distribution of the test data may differ significantly from that of the training data, which can lead to challenges in achieving effective generalization. Such mismatches can result in models that perform well on training data but struggle to make accurate predictions on test data. Over the past two decades, a variety of techniques have been developed to address the challenges associated with distribution shifts. These methods \citep{Gretton09,Bickel09,Blanchard21,Gizewski22} include domain adaptation and transfer learning, as well as robust statistical approaches designed to minimize the impact of discrepancies between training and test data distributions. In recent years, researchers \citep{Fang20} have explored both supervised and unsupervised strategies, incorporating advances in deep learning to enhance model resilience against such shifts. In particular, transfer learning \citep{Gretton09,Blanchard21,Ma21} focuses on using knowledge from related tasks or domains to improve generalization performance. Often, superior results can be achieved by pre-training on a larger, related dataset and then fine-tuning the model on the target dataset, as exemplified by LoRA for large language models \citep{LoRA}. A detailed overview and recent developments in this field can be found in \citet{Pan09} and references therein. These strategies play a crucial role in alleviating the challenges posed by distribution shifts in real-world machine learning applications. Addressing distribution shifts continues to be a critical area of research, with ongoing innovations aimed at improving model performance across diverse settings.

Within the transfer learning framework, two common scenarios are posterior drift and covariate shift. Posterior drift occurs when the conditional distribution of the output variable changes over time, even if the distribution of the input features remains constant. Covariate shift \citep{Cortes10,Gizewski22,Ma21,Gogolashvili2023,della2025computational}, which is the focus of this paper, arises when the distribution of input features changes between the training and test datasets, while the conditional distribution of the output variable given these inputs remains the same. Such shifts in the input distribution can significantly degrade the model's performance when applied to test data, as the patterns learned during training may no longer be applicable. Covariate shifts can occur in various real-world situations. For example, a model trained on data collected during a specific period may face challenges when applied to data from a different period, reflecting changes in trends or behaviors. Similarly, training on data from one geographic region and testing on data from another can lead to discrepancies due to regional variations in the underlying population or environmental factors.
In addition, shifts can occur due to changes in data collection processes or measurement techniques, which may introduce biases or modify the characteristics of the data. Addressing covariate shifts is essential for ensuring that machine learning models generalize well and remain robust in diverse applications. Various techniques \citep{Feng23, Gretton09,Bickel09,Nguyen10} have been developed to mitigate the effects of covariate shifts. Among these, reweighting methods are particularly important, as they modify the training data distribution to better match the test data. Specifically, importance weighting \citep{Cortes10, Fang20, Sugiyama07, Shimodaira00} adjusts the training data based on the likelihood that the input samples appear in the test dataset, effectively assigning more weight to instances that are more representative of the test distribution \citep{Kanamori09,Kanamori12,Sugi2012book,Zell23}.
% Domain adaptation aims to develop a model that can effectively transfer knowledge from the training domain to the test domain, which often involves techniques that align the feature distributions of both domains, such as kernel mean matching \cite{Gretton09, Bickel09}.
By employing these techniques, researchers aim to bridge the gap caused by covariate shifts, enhancing the model's capability to perform accurately across varying datasets and conditions.

In this paper, we consider nonparametric regression in the context of reproducing kernel Hilbert space (RKHS) \citep{Aron}. The RKHS  ${\cal H}_K$ is defined as the completion of the linear span of $\{K_x:x \in {\cal X}\}$ with the inner product denoted as $\langle \cdot, \cdot \rangle_K$ satisfying $\langle K_x, K_{x'} \rangle_K=K(x,x')$, where $\X$ is a separable and compact metric space. Here $K:{\cal X} \times {\cal X} \to \mathbb{R}$ be a Mercer kernel, i.e., a continuous, symmetric, positive semi-definite function. We say that $K$ is positive semi-definite, if for any finite set of points $\{s_1,\cdots,s_{\ell}\}\subset {\cal X}$  and any $\ell \in \mathbb{N}$, the matrix $\left(K(s_i, s_j)\right)_{i,j=1}^{\ell}$ is positive semi-definite. Let $K_x: {\cal X} \to \mathbb{R}$ be the function defined by $K_x(s)=K(x,s)$ for $x,s\in {\cal X}$.  Denote by $\|\cdot\|_K$ the norm of ${\cal H}_K$. It is well-known that the reproducing property
\begin{equation}\label{reproducingproperty}
	f(x)= \langle f, K_x \rangle_K
\end{equation} holds for all $x\in {\cal X}$ and $f\in {\cal H}_K$. Since ${\cal X}$ is compact, the space ${\cal H}_K$ is separable and contained in ${\cal C}({\cal X})$, i.e., the space of continuous functions on ${\cal X}$ with the norm $\|f\|_{\infty}=\sup_{x\in {\cal X}}|f(x)|$ and note, by the reproducing property (\ref{reproducingproperty}), for every $f\in {\cal H}_K$, that
\begin{equation}\label{normrelation1}
	\|f\|_{\infty} \leq \kappa \|f\|_K.
\end{equation} Here $\kappa=\sup_{x\in \mathcal{X}}\sqrt{K(x,x)}<\infty$ and we will always assume $\kappa \geq 1$ without loss of generality.

Given i.i.d. training samples ${\bf z}=\{(x_i,y_i)\}_{i=1}^n$ drawn from an unknown distribution $\rho^{tr}$ on $\mathcal{Z}:=\mathcal{X}\times\mathcal{Y}$, where the input space $\X$ is a separable and compact metric space, and $Y\in\mathcal{Y}=\mathbb{R}$ stands for the response variable and $\mathbb{E}[\cdot|X=x]$ is the conditional expectation with respect to $X=x$. The target of regression is to recover the regression function
\begin{equation*}
	f_\rho(x)=\int_\Y y d \rho(y|x),\quad \forall x\in\X,
\end{equation*}
using the training samples ${\bf z}$, where $\rho(y|x)$ is the conditional distribution of $\rho^{tr}$. Since  $\rho^{tr}$ is completely unknown and one attempts to learn a function $f_{\bz}$ as a good approximation of $f_{\rho}$. Taking the least square regression as an example, we define the generalization error as
\begin{equation*}
	\mathcal{E}_{\rho^{tr}}(f)=\int_{\X\times \Y}(f(x)-y)^2d\rho^{tr}(x,y).
\end{equation*}
Moreover, the regression function $f_\rho$ is the minimizer of the generalization error. In the standard least square regression, we usually assume the test samples are drawn from the same distribution as the training sample, and the performance of $f_{\bz}$ is usually measured by the excess generalization  error
\begin{equation*}
	\mathcal{E}_{\rho^{tr}}(f_\bz)-\mathcal{E}_{\rho^{tr}}(f_\rho)=\|f_{\bz}-f_\rho\|_{L_{\rho^{tr}_{\X}}^2}^2,
\end{equation*}
where $L^2_{\rho^{tr}_{\cal X}}$ be the Hilbert space of functions $f: {\cal X} \to {\cal Y}$ square-integrable with respect to the marginal distribution $\rho^{tr}_{\cal X}$ of $\rho^{tr}$. Denote by $\|\cdot\|_{\rho^{tr}_{\X}}$ the $L^2$ norm in the space $L^2_{\rho_\X^{tr}}$  induced by the inner product $\langle f, g\rangle_{\rho^{tr}_{\X}}=\int_{\cal X} f(x) g(x) d\rho_{\cal X}^{tr}(x)$ with $f, g\in L^2_{\rho^{tr}_{\cal X}}$.

% In the covariate shift setting, we assume the test samples are drawn from  distribution $\rho^{te}$ which is different from $\rho^{tr}$, but the conditional distribution are the same, that is,
% \begin{align*}
	% 	\rho^{tr}(x,y)=\rho(y|x)\rho_\X^{tr}(x),
	% \end{align*}
% and
% \begin{align*}
	% 	\rho^{te}(x,y)=\rho(y|x)\rho_\X^{te}(x).
	% \end{align*}
Distribution shifts are increasingly common in today's data-driven environments. Understanding and addressing these shifts is crucial for maintaining model accuracy and reliability in real-world applications.
% To mitigate distribution shifts, transfer learning has emerged as a crucial technique, particularly in situations where labeled data is scarce or expensive to acquire. This approach leverages knowledge from related tasks to enhance learning in new contexts, effectively bridging the gaps between different data distributions. This is especially relevant in cases of covariate shift, where the input feature distributions differ between training and test datasets, while the underlying relationships between inputs and outputs remain consistent.
In the scenario of covariate shift, test samples are drawn from a distribution $\rho^{te}$ that differs from the training distribution $\rho^{tr}$, while the conditional distributions stay unchanged, expressed as
\begin{align*}
	\rho^{tr}(x,y)=\rho(y|x)\rho_\X^{tr}(x),
\end{align*}
and
\begin{align*}
	\rho^{te}(x,y)=\rho(y|x)\rho_\X^{te}(x).
\end{align*}
% Transfer learning also addresses regression shift, where the relationship between inputs and outputs evolves over time, complicating predictions based on outdated models. In this case, the regression functions $f_{\rho^{tr}}$ and $f_{\rho^{te}}$ may differ between training and test datasets, even as the input feature distributions $\rho_\mathcal{X}^{tr}$ and $\rho_\mathcal{X}^{te}$ remain constant.

We define the prediction error as
\begin{align*} \mathcal{E}_{\rho^{te}}(f)=\mathbb{E}_{(x,y)\thicksim\rho^{te}}[(f(x)-y)^2]=\int_{\X\times\Y}(f(x)-y)^2d\rho^{te}(x,y).
\end{align*}
Our goal in covariate shift is to learn a function $f_\bz$ such that the prediction error $\mathcal{E}_{\rho^{te}}(f_\bz)$ is as small as possible, that is, we need to estimate the following excess prediction error
$$\mathcal{E}_{\rho^{te}}(f_{\bz})-\mathcal{E}_{\rho^{te}}(f_\rho)=\|f_{\bz}-f_\rho\|_{L_{\rho^{te}_{\X}}^2}^2.$$
One popular algorithm is the following weighted regularized least square algorithm (also known as weighted kernel ridge regression)
\begin{align}\label{algorithm: weighted KRR}
	f_{\bz, \lambda}^{ls}=\arg\min_{f\in\H_K} \frac{1}{n} \sum_{i=1}^n w(x_i) (f(x_i)-y_i)^2+\lambda\|f\|_K^2,
\end{align}
where $w(x)$ is the Radon-Nikodym derivative (also known as density ratio), which is defined as
\begin{align*}
	w(x)=\frac{d\rho_\X^{te}}{d\rho_\X^{tr}}(x).
\end{align*}
Recently, \citet{Ma21}  studied the algorithm (\ref{algorithm: weighted KRR}) under the assumption that
$w(\cdot)$ is either uniformly bounded or possesses a finite bounded second moment with respect to the training distribution. In contrast, \citet{Gogolashvili2023} investigated (\ref{algorithm: weighted KRR}) within the framework of covariate shift, addressing more general conditions on $w(\cdot)$ as specified in (\ref{inequality: weight assumption}). The solution to the algorithm (\ref{algorithm: weighted KRR}) can be expressed as
\begin{align*}
	f_{\bz, \lambda}^{ls}=(\lambda I+S_{X}^\top W S_X)^{-1}S_X^\top W\bar{y},
\end{align*}
where $S_X:\H_K\mapsto \mathbb{R}^n,$
$$ S_X f=(f(x_1), f(x_2),\cdots,f(x_n))^\top\in\mathbb{R}^n,$$
and
$	S_X^\top: \mathbb{R}^n \mapsto \H_K,$ with
\begin{align*}
	S_X^\top u(\cdot)=\frac{1}{n}\sum_{i=1}^n u_i K(\cdot,x_i) , u=(u_1, \cdots,u_n)^\top\in\mathbb{R}^n,
\end{align*}
and \begin{align*}
	W&=diag(w(x_1),\cdots,w(x_n)),\\
	\bar{y}&=(y_1,\cdots,y_n)^{\top}.
\end{align*}

In this paper, we consider a family of more general learning algorithms known as spectral algorithms, which were proposed to address ill-posed linear inverse problems (see, e.g., \citet{Engl96}) and have been employed in regression \citep{LoGerfo2008, Bauer07, guolinzhou2017,Gizewski22,Nguyen23,Fan24} by highlighting the connections between learning theory and inverse problems \citep{Devito2005}. The weighted spectral algorithm considered in this paper is of the form
\begin{align}\label{algorithm: weighted spectral algorithm}
	f_{\bz,\lambda}^{\mathbf{w}}=g_\lambda\left(S_{X}^\top W S_X\right)S_X^\top W\bar{y},
\end{align}
where the filter function $g_\lambda(\cdot)$ is defined in Definition \ref{definition: filter function}. When
$W=I,$ algorithm (\ref{algorithm: weighted spectral algorithm}) reduces to the classical spectral algorithm $f_{\bz,\lambda}=g_\lambda\left(S_{X}^\top S_X\right)S_X^\top \bar{y}$ as in \citet{Engl96}.
\begin{definition}\label{definition: filter function}
	We say that
	$g_\lambda:[0,U]\rightarrow\mathbb R,$ with $0 <
	\lambda\le U,$ is a filter function with {\sl qualification}
	$\nu_g \geq \frac{1}{2}$ if there exists a positive constant $b$
	independent of $\lambda$ such that
	\begin{equation}\label{condition1}
		\sup_{0<u\le U}|g_\lambda(u)|\le \frac{b}{\lambda},
		\qquad \sup_{0<u\le U}|g_\lambda(u)u|\le b,
	\end{equation}
	and
	\begin{equation}\label{condition2}
		\sup_{0<u\le U}|1-g_\lambda(u)u|u^\nu\le
		\gamma_\nu\lambda^\nu, \qquad \forall\ 0<\nu\leq
		{\nu}_g,
	\end{equation}
	where $\gamma_\nu>0$ is a constant depending only on
	$\nu\in(0,\nu_g],$  and $ U$ is some positive constant.
\end{definition}
If we choose the filter function $g_\lambda(u)=\frac{1}{\lambda+u},$ the corresponding estimator simplifies to the weighted regularized least square algorithm defined in \eqref{algorithm: weighted KRR}.  In this scenario, the constant $b=1,$ the qualification $\nu_g=1$ and the constant $\gamma_\nu=1.$ When $W=I,$ other examples of spectral algorithms with different filter functions include the Landweber iteration (gradient descent), defined by the filter function $g_\lambda(u)=\sum_{i=0}^{t-1}(1-u)^i$ with $\lambda=\frac{1}{t},$ $t\in\mathbb{N}$. Additionally, the spectral cutoff is induced by the filter function
\begin{equation*}
	g_\lambda(u)=\begin{cases}
		\frac{1}{u}, & \hbox{if} \ u\ge \lambda, \\
		0 & \hbox{if} \ u<\lambda.
	\end{cases}
\end{equation*}
For more examples of spectral algorithms and additional details, we refer readers to \citet{LoGerfo2008, Engl96, Bauer07, guolinzhou2017} and the references therein.  In this paper, we aim to estimate the excess prediction error
\begin{align}\label{eq: excess error}
	\mathcal{E}_{\rho^{te}}(f_{\bz,\lambda}^{\mathbf{w}})-	\mathcal{E}_{\rho^{te}}(f_{\rho})=\|f_{\bz,\lambda}^{\mathbf{w}}-f_\rho\|_{\rho_\X^{te}}^2.
\end{align}
Before presenting the main results, we give assumptions regarding the target function, weight function, and hypothesis space. This paper considers the following conditions for the weight function $w(\cdot)$ introduced in \citet{Gogolashvili2023}.
\begin{assumption}\label{assumption: weight assumption}
	There exist constants $\alpha\in[0,1]$, $C>0$ and $\sigma>0$ such that, for all $p\in\mathbb{N}$ with $p\ge 2$, it holds that
	\begin{align}\label{inequality: weight assumption}
		\left(\int_\X (w(x))^{\frac{p-1}{\alpha}} d\rho_{\X}^{te}(x)\right)^\alpha\le \frac12 p!C^{p-2} \sigma^2,
	\end{align}
	where the left hand side for $\alpha=0$ is defined as $\|w^{p-1}\|_{\infty,\rho_\X^{te}}$, the essential supremum of $w^{p-1}$ with respect to $\rho_\X^{te}.$
\end{assumption}

Assumption \ref{assumption: weight assumption} can be equivalently expressed as a condition on the R$\acute{e}$nyi divergence between $\rho_\X^{te}$ and $\rho_\X^{tr}$ \citep{Mansour09, Cortes10,Gogolashvili2023}. The R$\acute{e}$nyi divergence between $\rho_\X^{te}$ and $\rho_\X^{tr}$ with parameter $a \in (0, \infty]$ is defined as
\begin{equation*}
	H_a(\rho_\X^{te}\|\rho_\X^{tr}):=
	\begin{cases}
		a^{-1} \log \int_\X w(x)^a d\rho_\X^{te}(x) & (a > 0), \\
		\log(\|w\|_{\infty,\rho_\X^{te}}) & (a = \infty).
	\end{cases}
\end{equation*}
Under Assumption \ref{assumption: weight assumption}, for all integers $p \geq 2$, the R$\acute{e}$nyi divergence  must satisfy the following upper bound
\begin{equation*}
	H_{(p-1)/\alpha}(\rho_\X^{te}\|\rho_\X^{tr}) \leq \frac{1}{p-1} \left( \log p! + \log \left( \frac{C^{p-2}\sigma^2}{2} \right) \right).
\end{equation*}
Intuitively, Assumption \ref{assumption: weight assumption} ensures that the testing distribution $\rho_\X^{te}$ remains close to the training distribution $\rho_\X^{tr}(x)$, and the parameter $\alpha \in [0, 1]$ controls the allowable deviation \citep{Gogolashvili2023}. Notably, when $\alpha=1$, the assumption guarantees that all moments of the weight function $w(\cdot)$ with respect to the testing distribution $\rho_\X^{te}$ are finite.

Define the integral operator  $L_K:L^2_{\rho_{\X}^{te}}\rightarrow L^2_{\rho_{\X}^{te}}$ on $\mathcal{H}_K$ or $L^2_{\rho_{\X}^{te}}$ associated with the Mercer kernel
$K$ by
\begin{align*}
	L_K f=\int_\X f(x)K_xd\rho_\X^{te}(x), f\in L^2_{\rho_{\X}^{te}}.
\end{align*}
Next we introduce our assumption regarding the regularity (often interpreted as smoothness) of the regression function $f_\rho.$
\begin{assumption}\label{assumption: regularity condition}
	\begin{equation}\label{regularitycondition}
		f_\rho=L_K^r (u_\rho) ~~{\rm for~some}~r>0~{\rm and} ~ u_\rho\in L^2_{\rho_{\X}^{te}},
	\end{equation}
	where $L_K^r$ denotes the $r$-th power of $L_K$ on $L^2_{\rho_{\X}^{te}}$
	since $L_K:L^2_{\rho_{\X}^{te}}\rightarrow L^2_{\rho_{\X}^{te}}$ is a compact  and
	positive operator.
\end{assumption}
This assumption is standard in learning theory and can be further interpreted through the theory of interpolation spaces \citep{Smale2003}. Moreover, since $\rho_\X^{te}$ is non-degenerate, Theorem 4.12 in \citet{cucker2007learning} implies that $L^{1/2}_K$ is an isomorphism from $\overline{{\cal H}_K}$, the closure of ${\cal H}_K$ in $L^2_{\rho_\X^{te}}$, to ${\cal H}_K$. That is, for every $f\in \overline{{\cal H}_K}$, we have $L^{1/2}_K f \in {\cal H}_K$ and
\begin{equation}\label{normrelation2}
	\|f\|_{\rho_\X^{te}}=\left\|L^{1/2}_K f\right\|_K.
\end{equation} Therefore, $L^{1/2}_{K}(L^2_{\rho_\X^{te}})={\cal H}_{K}$, and when $r\ge\frac12$, condition (\ref{regularitycondition}) ensures $f_\rho\in\H_K $.

We shall use the {\it effective dimension} $\mathcal{N}(\lambda)$ to measure the complexity of $\mathcal H_K$ with respect to $\rho_\X^{te},$ which is defined to be the trace of the operator $(\lambda I+L_K)^{-1}L_K,$ that is
\[
\mathcal{N}(\lambda)={\rm Tr}((\lambda I+L_K)^{-1}L_K),  \qquad \lambda>0.
\]
\begin{assumption}\label{assumption: effective dimension}
	There exist	a parameter $0<\beta\leq 1$ and a constant $C_0>0$ such that
	\begin{equation}\label{Assumption on effecdim}
		\mathcal N(\lambda)\leq C_0\lambda^{-\beta}, \qquad  \forall \lambda>0.
	\end{equation}
\end{assumption}

The condition (\ref{Assumption on effecdim}) with $\beta=1$ is always satisfied by taking the constant
$C_0=\mbox{Tr}(L_K)\leq\kappa^2$. The capacity of the hypothesis space $\mathcal{H}_K$  is commonly characterized by covering number, effective dimension and eigenvalue decay conditions of the integral operator $L_K.$ It has been demonstrated in \citet{GuoGuoShi2023} that Assumption \ref{assumption: effective dimension} with $0<\beta<1$ is equivalent to $\lambda_i(L_K)=\mathcal{O}(i^{-1/\beta})$, where $\{\lambda_i(L_K)\}_{i=1}^\infty$ of $L_K$ are the eigenvalues arranged in non-increasing order. Here, we remark that if $L_K$ is of finite rank, i.e., the range of $L_K$ is finite-dimensional, we will set $\beta=0$.

In practice, the weight function $w(\cdot)$ is often unbounded \citep{Cortes10,Ma21}, which poses unique challenges in their applications and error analysis. When working with unnormalized weights, as shown in \eqref{algorithm: weighted KRR}, the standard analysis used for kernel ridge regression is not directly applicable to spectral algorithms like gradient descent. This limitation arises because the behavior of unnormalized weights can significantly affect the convergence and stability of these algorithms. We first consider the weighted spectral algorithm with normalized weight \citep{Cortes10}, defined as
\begin{equation}\label{algorithm: spectral algorithm with normalized weight}
	f_{\bz,\lambda}^{\overline{\mathbf{w}}}=g_\lambda\left(S_{X}^\top \overline{W} S_X\right)S_X^\top \overline{W}\bar{y},
\end{equation}
where $\overline{W}=diag(\bar{w}(x_1),\cdots,\bar{w}(x_n))$, and the normalized weight $\bar{w}(x_i)$ is given by
\begin{equation}\label{normalized weight}
	\bar{w}(x_i)=\frac{w(x_i)}{\frac{1}{n}\sum_{j=1}^n w(x_j)}.
\end{equation}
From this definition, it follows immediately that $\frac{1}{n}\sum_{i=1}^n \bar{w}(x_i)=1,$ and $0\le \bar{w}(x_i)\le n.$ This advantage comes at the cost of introducing bias in the weights. However, we show that, under mild conditions, normalized and unnormalized weights remain very close with high probability.
\begin{theorem}\label{theorem: main result with normalized weight}
	Let the weighted spectral algorithm with normalized weight be defined by (\ref{algorithm: spectral algorithm with normalized weight}). Under Assumption \ref{assumption: weight assumption} with $0<\alpha\le 1,$  Assumption \ref{assumption: regularity condition} with $1/2\le r\le \nu_g$, and Assumption \ref{assumption: effective dimension} with $0<\beta\le 1$,
	if we choose $\lambda=n^{-\frac{1}{\min\{2r,3\}+\beta+\alpha(1-\beta)}}$ with $\beta+\alpha(1-\beta)\ge 1,$ then  for any $0<\delta<1$, with confidence at least $1-
	\delta,$ we have
	\begin{equation}
		\|f_{\bz,\lambda}^{\overline{\mathbf{w}}}-f_\rho\|_{\rho_\X^{te}}\le C' n^{-\frac{\min\{r,3/2\}}{\min\{2r,3\}+\beta+\alpha(1-\beta)}}  \left(\log\frac{12}{\delta}\right)^3,
	\end{equation}	
	where the constant $C'$ is independent of $n$ or $\delta$ and will be given in the proof.	
\end{theorem}
If we choose $\beta=1$ (which corresponds to the capacity-independent case), we obtain optimal capacity independent convergence rates when $1/2\le r\le \min\{\nu_g, 3/2\}$, even when the weight function $w(\cdot)$ is potentially unbounded.
\begin{corollary}\label{corollary: weightd spectral algorithm with normalized weight}
	Let the weighted spectral algorithm with  normalized weight be defined by (\ref{algorithm: spectral algorithm with normalized weight}), under Assumption \ref{assumption: weight assumption} with $0<\alpha\le 1,$  Assumption \ref{assumption: regularity condition} with $1/2\le r\le \nu_g$, if we take  $\lambda=n^{-\frac{1}{\min\{2r,3\}+1}},$  then with confidence at least $1-\delta,$ there holds
	\begin{equation}
		\|f_{\bz,\lambda}^{\overline{\mathbf{w}}}-f_\rho\|_{\rho_\X^{te}}\le C' n^{-\frac{\min\{r,3/2\}}{\min\{2r,3\}+1}}  \left(\log\frac{12}{\delta}\right)^3.
	\end{equation}	
\end{corollary}
\noindent{\bf Remark:} The result in Corollary \ref{corollary: weightd spectral algorithm with normalized weight} shows that the convergence rate achieved by the weighted spectral algorithm with normalized weights is capacity-independent optimal for $1/2\le r\le \min\{\nu_g, 3/2\}$, even in the presence of unbounded weight functions $w(\cdot)$. This represents a nontrival extension of the results for the weighted kernel ridge regression \eqref{algorithm: weighted KRR} presented in \citet{Gogolashvili2023}. However, Theorem \ref{theorem: main result with normalized weight} and Corollary \ref{corollary: weightd spectral algorithm with normalized weight} reveal that these rates suffer from saturation effects when the regularity parameter satisfies $r> 3/2$, even for spectral algorithms with an infinite qualification $\nu_g$. This limitation stems from the $\mathcal{O}(n)$ uniform bound on the normalized weights $\bar{w}(\cdot)$, which causes the operator norm of the empirical integral operator $S_{X}^\top \overline{W} S_X$ to scale linearly with the sample size $n$. This motivates us to consider a clipped weight $\hat{w}$, which effectively and adaptively handles unbounded weights while avoiding the saturation phenomenon.

Our second main result establishes an upper bound for $\|f_{\bz,\lambda}^{\hat{\bf w} }-f_\rho\|_{\rho_\X^{te}}$ when the weight function $w(\cdot)$ is potentially unbounded and satisfies Assumption \ref{assumption: weight assumption} with $0<\alpha\le 1.$ Here the  weighted spectral algorithm  with clipped weight, denoted by $f_{\bz,\lambda}^{\hat{\bf w} }$, is given by
%Our first main result (Theorem \ref{theorem: clipped weight spectral algorithm}) reveals that the weighted spectral algorithm
\begin{align}\label{algorithm: spectral algorithm with clipped weight}
	f_{\bz,\lambda}^{\hat{\bf w} }=g_\lambda\left(S_{X}^\top \hat{W}S_X\right)S_X^\top\hat{W} \bar{y},
\end{align}
where  $\hat{W}=diag(\hat{w}(x_1),\cdots,\hat{w}(x_n))$, and the clipped weight $\hat{w}(x_i)$ is defined as
\begin{equation}\label{eq: clipped weight}
	\hat{w}(x_i) = \left\{\begin{array}{ll} w(x_i), & \hbox{when} \ w(x_i)<D_n, \\
		D_n, & \hbox{when} \
		w(x_i)\ge D_n. \end{array}\right.
\end{equation}
Here $D_n$ is a parameter to be determined. Without loss of generality, we assume $D_n>1$.

\begin{theorem}\label{theorem: clipped weight spectral algorithm}
	Let the spectral algorithm with clipped weight be defined by (\ref{algorithm: spectral algorithm with clipped weight}), under Assumption \ref{assumption: weight assumption} with $0<\alpha\le 1,$  Assumption \ref{assumption: regularity condition} with $1/2\le r\le \nu_g$, and Assumption \ref{assumption: effective dimension} with $0<\beta\le 1$,
	if we take $\lambda=n^{-\frac{1}{2r+\beta}+\frac{\epsilon}{r} }$ for any $0<\epsilon<\frac{r}{2r+\beta}$, and
	\begin{equation*}
		D_n = \left\{\begin{array}{ll} n^{\alpha \epsilon}, & \hbox{when} \ 1/2\le r\le 3/2, \\
			n^{\frac{\alpha\epsilon}{r-1/2}}, & \hbox{when} \
			r>3/2, \end{array}\right.
	\end{equation*} then for $0<\delta<1$, with confidence at least $1-\delta,$ there holds
	\begin{equation*}
		\|f_{\bz,\lambda}^{\hat{\bf w} }-f_\rho\|_{\rho_\X^{te}}\le \tilde{C}_{r,\epsilon} n^{-\frac{r}{2r+\beta}+\epsilon} \left(\log\frac{6}{\delta}\right)^2,
	\end{equation*}
	where the constant $\tilde{C}_{r,\epsilon}$ is independent of the sample size $n$ or $\delta$ and will be given in the proof.
\end{theorem}
Our results in Theorem \ref{theorem: clipped weight spectral algorithm} demonstrate that the convergence rates of weighted spectral algorithm with clipped weight (\ref{algorithm: spectral algorithm with clipped weight})  can approach the optimal capacity-dependent rates arbitrarily closely, when the weight function is potentially unbounded. Notably, the parameter $\alpha$ only influences the truncation threshold $D_n$ and has no impact on the convergence rates. When the weight function is uniformly bounded, from the proof of Theorem \ref{theorem: clipped weight spectral algorithm}, the weighted spectral algorithm (\ref{algorithm: spectral algorithm with clipped weight}) achieves the optimal capacity-dependent convergence rate. Moreover, as shown in Theorem \ref{theorem: unweighted spectral algorithm}, the weighting becomes unnecessary in this case, the classical (unweighted) spectral algorithm can also achieve the optimal rates under mild assumptions. This leads to our third main result.

To establish the main results for the unweighted spectral algorithms, we require the following mild assumptions on the target function
$f_\rho$ and the hypothesis space $\mathcal{H}_K.$

%This leads to our third main result: for uniformly bounded weights, we establish that the unweighted spectral algorithm achieves optimal capacity-dependent convergence rates under mild conditions.	
\begin{assumption}\label{assumption: regularity condition2}
	\begin{equation}\label{regularitycondition2}
		f_\rho=\tilde{L}_K^{\tilde{r}} (v_\rho) ~~{\rm for~some}~{\tilde{r}}>0~{\rm and} ~ v_\rho\in L^2_{\rho_{\X}^{tr}},
	\end{equation}
	where $\tilde{L}_K^{\tilde{r}}$ denotes the $\tilde{r}$-th power of $\tilde{L}_K$ on $L^2_{\rho_{\X}^{tr}}$
	since $\tilde{L}_K:L^2_{\rho_{\X}^{tr}}\rightarrow L^2_{\rho_{\X}^{tr}}$ is a compact  and
	positive operator.
\end{assumption}
We shall use the {\it effective dimension} $\tilde{\mathcal{N}}(\lambda)$ to
measure the complexity of $\mathcal H_K$ with respect to $\rho_\X^{tr}.$
\begin{assumption}\label{assumption: effective dimension2}
	There exist	a parameter $0<\tilde{\beta}\leq 1$ and a constant $C_0>0$ as
	\begin{equation}\label{Assumption on effecdim2}
		\tilde{\mathcal{N}}(\lambda)={\rm Tr}((\lambda I+\tilde{L}_K)^{-1}\tilde{L}_K)\leq C_0\lambda^{-\tilde{\beta}}, \qquad  \forall \lambda>0.
	\end{equation}
\end{assumption}

\begin{theorem}\label{theorem: unweighted spectral algorithm}
	Let the unweighted spectral algorithm be defined by $f_{\bz,\lambda}=g_\lambda(S_{X}^\top S_X)S_X^\top \bar{y}$, under Assumption \ref{assumption: regularity condition2} with $1/2\le \tilde{r}\le \nu_g$, and Assumption \ref{assumption: effective dimension2} with $0<\tilde{\beta}\le 1$,
	if we take $\lambda=n^{-\frac{1}{2\tilde{r}+\tilde{\beta}}}$, then for $0<\delta<1$, with confidence at least $1-\delta$, there holds
	\begin{equation*}
		\|f_{\bz,\lambda}-f_\rho\|_{\rho_\X^{te}}\le \tilde{C}_{\tilde{r}} n^{-\frac{\tilde{r}}{2\tilde{r}+\tilde{\beta}}} \left(\log\frac{6}{\delta}\right)^4,
	\end{equation*}
	where the constant $\tilde{C}_{\tilde{r}}$ is independent of the sample size $n$ and will be given in the proof.
\end{theorem}
The remainder of this paper is organized as follows. We discuss related work in Section \ref{section: related work} and conduct error decomposition in Section \ref{section: error decomposition}. The proofs of our main results are provided in Section \ref{section: proof of main results}.
%    Before proving the main results, we present some preliminary results in Section \ref{section: preliminary}.
\section{Related Work and Discussion}\label{section: related work}

Although the literature on covariate shift is extensive, we focus specifically on the algorithms most relevant to our work. Recently, the work presented in \citet{Ma21} explored the implications of kernel ridge regression (KRR) in the context of covariate shift. The authors demonstrated that when the density ratios are uniformly bounded, the KRR estimator can achieve a minimax optimal convergence rate, even without complete knowledge about the density ratio, only an upper bound is necessary for practical application. Furthermore, they investigated a wider array of covariate shift issues where the density ratio may not be bounded but possesses a finite second moment for the training distribution. In such scenarios, they studied an importance-weighted KRR estimator that modifies sample weights by carefully truncating the density ratios, which maintains minimax optimality. Their findings also emphasized that in situations characterized by model misspecification, employing importance weighting can lead to a more accurate approximation of the regression function for the distribution of the test inputs. Recently, \citet{della2025computational} explored the computational efficiency of KRR in the context of covariate shift through the use of random projections. \citet{Feng23} extends the results of \citet{Ma21} to a broader class of learning algorithms with general convex loss functions, establishing sharp convergence rates under the same covariate shift assumptions as in \citet{Ma21}.

However, the assumptions made in \citet{Ma21}, particularly regarding the uniform boundedness of the eigenfunctions of the kernel integral operator, can be challenging to verify in practical applications. In response, the authors of \citet{Gogolashvili2023} relaxed these assumptions regarding density ratios and eigenfunctions discussed in \citet{Ma21}. Their analysis of the importance-weighted KRR (\ref{algorithm: weighted KRR}) encompassed a broader range of contexts, including both parametric and nonparametric models, as well as cases of model specification and misspecification, thereby allowing for arbitrary weighting functions. These comprehensive studies significantly enhanced the understanding of how to effectively implement importance weighting across various modeling scenarios.

Despite these advancements, the algorithms proposed in \citet{Ma21} and \citet{Gogolashvili2023} confront the saturation effect, where improvements in the learning rate stabilize and cease to increase effectively once the regression function attains a certain level of regularity. To address this issue,  \citet{Gizewski22} introduced a generalized regularization framework for covariate shifts via weighted spectral algorithms. Their analysis establishes capacity-independent learning rates for the resulting regularized estimators, extending the guarantees of importance-weighted KRR while requiring the bounded density ratio assumption.

This work presents a solid theoretical analysis of spectral algorithms under covariate shifts, with kernel ridge regression (KRR) as a canonical special case. Our analysis removes the restrictive bounded eigenfunction condition required in \citet{Ma21}, thereby significantly expanding the theoretical applicability of spectral methods. Our main contributions are threefold: First, Theorem \ref{theorem: unweighted spectral algorithm}  demonstrates that a uniform bound on the density ratio suffices to attain minimax optimal convergence rates for the unweighted spectral algorithms. Second, for the more challenging case of potentially unbounded density ratios $w(\cdot)$, we develop a novel weighted spectral algorithm with normalized weights. Corollary \ref{corollary: weightd spectral algorithm with normalized weight} proves that this framework achieves capacity-independent optimal rates, even when $w(\cdot)$ is unbounded. However, these rates exhibit saturation effects when the regularity parameter $r > 3/2$, due to the $\mathcal{O}(n)$ uniform bound on normalized weights $\bar{w}(\cdot)$. To overcome saturation limitations, furthermore, we propose a truncated weight method. Theorem \ref{theorem: clipped weight spectral algorithm} demonstrates that spectral algorithms with clipped weights can approximate capacity-dependent optimal rates arbitrarily closely. These theoretical advances significantly expand the scope of existing approaches and provide fundamental insights for optimizing learning under covariate shifts.

\section{Error Decomposition}\label{section: error decomposition}
In this section, we consider the error decomposition when the regression function $f_\rho$ satisfies condition (\ref{regularitycondition}) with $r\ge \frac12,$ which implies $f_\rho\in \mathcal{H}_K$. We begin with some useful lemmas.

First, we  establish the following lemma, which corresponds to Lemma 5 in \citet{guolinzhou2017}, based on the properties of the filter function. The result also holds if $\overline{W}$ is replaced by $\hat{W}.$
\begin{lemma}\label{lemma: property of filter function}
	Let $g_\lambda$ be defined by (\ref{definition: filter function}), for $0<t\le \nu_g,$ we have \begin{equation*}
		\left\|\left(g_\lambda\left(S_{X}^\top \overline{W} S_X\right)S_{X}^\top \overline{W} S_X-I\right)\left(\lambda I+S_{X}^\top \overline{W} S_X\right)^t\right\|_{op}\le
		2^t(b+1+\gamma_t)\lambda^t.
	\end{equation*}
\end{lemma}
%\begin{proof}
%	Recall from the introduction that $\{(\sigma_i^{\bf
	%		x},\phi_i^{\bf x})_{i}\}$ is a set of normalized eigenpairs of
%	$S_{X}^\top W S_X$ with the eigenfunctions $\{\phi_i^{\bf x}\}_i$ forming an
%	orthonormal basis of $\mathcal H_K$. Then for each $h\in \mathcal
%	H_K,$ we have $h=\sum\langle h,\phi_i^{\bf x}\rangle\phi_i^{\bf x}$
%	and $\|h\|_K^2=\sum_{i}\langle h,\phi_i^{\bf x}\rangle^2.$ Hence
%	\begin{eqnarray*}
	%		&&\|(g_\lambda(S_{X}^\top W S_X)S_{X}^\top W S_X-I)(\lambda I+S_{X}^\top W S_X)^t h\|_K \\
	%		&=&\left\|\sum_{i=1}^\infty\langle h,\phi_i^{\bf x}\rangle
	%		(g_\lambda(\sigma_i^{\bf x})\sigma_i^{\bf x}-1)(\lambda+\sigma_i^{\bf x})^t\phi_i^{\bf x}\right\|_K\\
	%		&=&\left\{\sum_{i=1}^\infty\left[\langle h,\phi_i^{\bf x}\rangle (g_\lambda(\sigma_i^{\bf x})\sigma_i^{\bf x}-1)(\lambda+\sigma_i^{\bf x})^t\right]^2\right\}^{1/2}\\
	%		&\le& \left\{\sum_{i=1}^\infty\left[|\langle h,\phi_i^{\bf x}\rangle| |g_\lambda(\sigma_i^{\bf x})\sigma_i^{\bf x}-1| 2^t (\lambda^t+(\sigma_i^{\bf x})^t)\right]^2\right\}^{1/2}\\
	%		&\le& 2^t (b+1+\gamma_t)\lambda^t \left\{\sum_{i=1}^\infty(\langle
	%		h,\phi_i^{\bf x}\rangle)^2 \right\}^{1/2} = 2^t (b+1+\gamma_t)\lambda^t \|h\|_K.
	%	\end{eqnarray*}
%	The last inequalities hold due to properties (\ref{condition1}) and
%	(\ref{condition2}) of the filter function $g_\lambda$ and the
%	elementary inequality $(c+d)^t\le 2^t(c^t+d^t)$ for any $c,d,t\ge0.$
%	Then our result follows from the definition of the operator norm on
%	$\mathcal H_K.$
%\end{proof}

The following Cordes inequality was proved in \citet{Bathia1997} for positive definite matrices and later presented  in \citet{Blanchard2010} for positive operators in Hilbert spaces.
\begin{lemma}\label{lemma: Cordes inequality}
	Let $s\in[0,1].$ For positive operators $A$ and $B$
	on a Hilbert space we have
	\begin{equation}\label{multiply}
		\|A^sB^s\|_{op}\le\|AB\|_{op}^s.
	\end{equation}
\end{lemma}
We also need the following lemma in our error decomposition, which can be found in \citet{Blanchard2010}.
\begin{lemma}\label{lemma: operator difference}
	For positive   operators $A$ and $B$ on a Hilbert
	space with $\|A\|_{op},\|B\|_{op}\le \mathcal C$ for some constant $\mathcal
	C>0,$ we have for $t\ge 1,$
	\begin{equation*}
		\|A^t-B^t\|_{op}\le t\mathcal C^{t-1}\|A-B\|_{op}.
	\end{equation*}
\end{lemma}
To analyze the error of weighted spectral algorithm with normalized weight (\ref{algorithm: spectral algorithm with normalized weight}), we need the following error decomposition.
\begin{proposition}\label{proposition: error decomposition for normalization}
	Let $f_{\bz,\lambda}^{\overline{\mathbf{w}}}$ be defined by \eqref{algorithm: spectral algorithm with normalized weight}. Suppose Assumption \ref{assumption: effective dimension}  holds with
	$1/2\le r\le \nu_g$.
	When $1/2\le r\le 3/2$, we have
	\begin{align*}
		\|f_{\bz,\lambda}^{\overline{\mathbf{w}}}-f_\rho\|_{\rho_\X^{te}}&\le
		2b \left\|\left(\lambda I+ S_{X}^\top \overline{W} S_X\right)^{-1} \left(\lambda I+ L_K\right)\right\|_{op}\cdot \left\| \left(\lambda I+ L_K\right)^{-1/2} \left(S_X^\top \overline{W}\bar{y}-S_X^\top \overline{W} S_X f_\rho\right)\right\|_K\\
		&+ C_r\lambda^r \left\|\left(\lambda I+ S_{X}^\top \overline{W} S_X\right)^{-1}(\lambda I+L_K)\right\|_{op}^{r}.
	\end{align*}
	When $r>3/2,$ we have
	\begin{align*}
		\|f_{\bz,\lambda}^{\overline{\mathbf{w}}}-f_\rho\|_{\rho_\X^{te}}&\le
		2b \left\|\left(\lambda I+ S_{X}^\top \overline{W} S_X\right)^{-1} \left(\lambda I+ L_K\right)\right\|_{op}\cdot \left\| (\lambda I+ L_K)^{-1/2} \left(S_X^\top \overline{W}\bar{y}-S_X^\top \overline{W} S_X f_\rho\right)\right\|_K\\
		&+ C_r \left\|\left(\lambda I+ S_{X}^\top \overline{W} S_X\right)^{-1} (\lambda I+ L_K)\right\|_{op}^{1/2}\left(\lambda^{1/2} \left\|L_K-S_{X}^\top \overline{W} S_X\right\|_{op}+\lambda^{3/2}\right),
	\end{align*}
	where $C_r=\max\left\{2^r(b+1+\gamma_r),\left(\sqrt{2}(b+1+\gamma_{1/2})
	+2^{3/2}(b+1+\gamma_{3/2})\right)\kappa^{2r-3} \right\}\left\| u_\rho\right\|_{\rho_\X^{te}}.$
\end{proposition}
\begin{proof}
	First, by the definition (\ref{algorithm: spectral algorithm with normalized weight}) of $f_{\bz,\lambda}^{\overline{\mathbf{w}}},$ we obtain
	\begin{align*}
		f_{\bz,\lambda}^{\overline{\mathbf{w}}}-f_\rho&=g_\lambda\left(S_{X}^\top \overline{W} S_X\right)S_X^\top \overline{W}\bar{y}-f_\rho\\&=g_\lambda\left(S_{X}^\top \overline{W} S_X\right)S_X^\top \overline{W}\bar{y}-g_\lambda\left(S_{X}^\top \overline{W} S_X\right)S_X^\top \overline{W} S_X f_\rho +g_\lambda\left(S_{X}^\top \overline{W} S_X\right)S_X^\top \overline{W} S_X f_\rho-f_\rho\\
		&=g_\lambda\left(S_{X}^\top \overline{W} S_X\right)\left(S_X^\top \overline{W}\bar{y}-S_X^\top \overline{W} S_X f_\rho\right) +\left(g_\lambda\left(S_{X}^\top \overline{W} S_X\right)S_X^\top \overline{W} S_X -I\right)f_\rho.
	\end{align*}
	Then we see from the
	identity $\|f\|_{\rho_\X^{te}} = \left\|L_K^{1/2} f\right\|_K$ for $f\in {\mathcal H}_K$ and $\left\|L_K^{1/2}(\lambda I+L_K)^{-1/2}\right\|_{op}\leq
	1$ that
	\begin{equation}\label{equation?error decomposition intermediate term}
		\begin{aligned}
			\left\|f_{\bz,\lambda}^{\overline{\mathbf{w}}}-f_\rho\right\|_{\rho_\X^{te}}&=\left\|L_K^{1/2}(f_{\bz,\lambda}^{\overline{\mathbf{w}}}-f_\rho)\right\|_K\le \left\|(\lambda I+ L_K)^{1/2}(f_{\bz,\lambda}^{\overline{\mathbf{w}}}-f_\rho)\right\|_K\\
			&=\left\|(\lambda I+ L_K)^{1/2} \left(\lambda I+ S_{X}^\top \overline{W} S_X\right)^{-1/2} \left(\lambda I+ S_{X}^\top \overline{W} S_X\right)^{1/2} (f_{\bz,\lambda}^{\mathbf{w}}-f_\rho)\right\|_K\\
			&\le \left\|(\lambda I+ L_K)^{1/2}  \left(\lambda I+ S_{X}^\top \overline{W} S_X\right)^{-1/2}\right\|_{op} \cdot\\
			&~~~~\Big[\left\|\left(\lambda I+ S_{X}^\top \overline{W} S_X\right)^{1/2} g_\lambda\left(S_{X}^\top \overline{W} S_X\right)\left(S_X^\top W\bar{y}-S_X^\top \overline{W} S_X f_\rho\right)\right\|_K \\
			&~~~~+ \left\|\left(\lambda I+ S_{X}^\top \overline{W} S_X\right)^{1/2} \left(g_\lambda\left(S_{X}^\top \overline{W} S_X\right)S_X^\top \overline{W} S_X -I\right)f_\rho\right\|_K\Big]\\
			&=:I_1(I_2+I_3),
		\end{aligned}
	\end{equation}
	where
	\begin{align*}
		I_1&=\left\|(\lambda I+ L_K)^{1/2}  \left(\lambda I+ S_{X}^\top \overline{W} S_X\right)^{-1/2}\right\|_{op},\\
		I_2&=\left\|\left(\lambda I+ S_{X}^\top \overline{W} S_X\right)^{1/2} g_\lambda\left(S_{X}^\top \overline{W} S_X\right)(S_X^\top \overline{W} \bar{y}-S_X^\top \overline{W} S_X f_\rho)\right\|_K,\\
		I_3&=\left\|\left(\lambda I+ S_{X}^\top \overline{W} S_X\right)^{1/2} \left(g_\lambda\left(S_{X}^\top \overline{W} S_X\right)S_X^\top \overline{W} S_X -I\right)f_\rho\right\|_K.
	\end{align*}
	In the following, we will estimate the above three terms.
	
	For the first term $I_1,$ by Lemma \ref{lemma: Cordes inequality} with $s=1/2,$ it follows that
	\begin{align*}
		I_1=\left\|(\lambda I+ L_K)^{1/2}  \left(\lambda I+ S_{X}^\top \overline{W} S_X\right)^{-1/2}\right\|_{op}\le \left\|(\lambda I+ L_K) \left(\lambda I+ S_{X}^\top \overline{W} S_X\right)^{-1}\right\|_{op}^{1/2}.
	\end{align*}
	For the second term $I_2$, we have
	\begin{align}
		I_2&=\left\|\left(\lambda I+ S_{X}^\top \overline{W} S_X\right)^{1/2}g_\lambda\left(S_{X}^\top \overline{W} S_X\right)\left(S_X^\top \overline{W}\bar{y}-S_X^\top \overline{W} S_X f_\rho\right)\right\|_K \notag\\
		&=\Big\|\left(\lambda I+ S_{X}^\top \overline{W} S_X\right)^{1/2} g_\lambda\left(S_{X}^\top \overline{W} S_X\right) \left(\lambda I+ S_{X}^\top \overline{W} S_X\right)^{1/2} \left(\lambda I+ S_{X}^\top \overline{W} S_X\right)^{-1/2} (\lambda I+ L_K)^{1/2} \notag\\
		&~~~\cdot (\lambda I+ L_K)^{-1/2} \left(S_X^\top \overline{W}\bar{y}-S_X^\top \overline{W} S_X f_\rho\right)\Big\|_K \notag\\
		&\le \left\|\left(\lambda I+ S_{X}^\top \overline{W} S_X\right) g_\lambda\left(S_{X}^\top \overline{W} S_X\right) \right\|_{op} \cdot \left\| \left(\lambda I+ S_{X}^\top \overline{W} S_X\right)^{-1/2} (\lambda I+ L_K)^{1/2}\right\|_{op} \notag\\
		&~~~\cdot \left\|(\lambda I+ L_K)^{-1/2} \left(S_X^\top \overline{W}\bar{y}-S_X^\top \overline{W} S_X f_\rho\right)\right\|_K\notag\\
		&\le 2b \left\|\left(\lambda I+ S_{X}^\top \overline{W} S_X\right)^{-1} (\lambda I+ L_K)\right\|_{op}^{1/2}\cdot \left\| (\lambda I+ L_K)^{-1/2} \left(S_X^\top \overline{W}\bar{y}-S_X^\top \overline{W} S_X f_\rho\right)\right\|_K. \notag
	\end{align}
	where the last inequality holds due to the property
	(\ref{condition1}) of the filter function $g_\lambda$ and the Cordes inequality with $s=1/2$ in Lemma \ref{lemma: Cordes inequality}.
	
	For the third term $I_3$, since $f_\rho=L_K^r u_\rho $ with $u_\rho\in L_{\rho_\X^{te}}^2$ and $r\ge \frac12$, there holds
	\begin{align}
		I_3&=\left\|\left(\lambda I+ S_{X}^\top \overline{W} S_X\right)^{1/2} \left(g_\lambda\left(S_{X}^\top \overline{W} S_X\right)S_X^\top \overline{W} S_X -I\right)f_\rho\right\|_K \notag\\
		&=\left\|\left(\lambda I+ S_{X}^\top \overline{W} S_X\right)^{1/2}\left(g_\lambda\left(S_{X}^\top \overline{W} S_X\right)S_X^\top \overline{W} S_X -I\right)L_K^r u_\rho\right\|_K \notag\\
		&\le \left\|\left(\lambda I+ S_{X}^\top \overline{W} S_X\right)^{1/2}\left(g_\lambda\left(S_{X}^\top \overline{W} S_X\right)S_X^\top \overline{W} S_X -I\right)L_K^{r-1/2}\right\|_{op} \left\| u_\rho\right\|_{\rho_\X^{te}}. \notag
	\end{align}
	To estimate the term $\left\|\left(\lambda I+ S_{X}^\top \overline{W} S_X\right)^{1/2}\left(g_\lambda\left(S_{X}^\top \overline{W} S_X\right)S_X^\top \overline{W} S_X -I\right)L_K^{r-1/2}\right\|_{op}$, we will consider two cases based on the regularity of $f_\rho$. \\
	Case 1: $1/2 \le r\le 3/2$.\\
	In this case, $0\le r-1/2\le 1,$ we rewrite $L_K^{r-1/2}$ as
	\begin{align*}
		L_K^{r-1/2}=\left(\lambda I+ S_{X}^\top \overline{W} S_X\right)^{r-1/2}\left(\lambda I+ S_{X}^\top \overline{W} S_X\right)^{-(r-1/2)}(\lambda I+L_K)^{r-1/2}(\lambda I+L_K)^{-(r-1/2)}L_K^{r-1/2}.
	\end{align*}
	Then we have
	\begin{align*}
		I_3&\le \left\|\left(\lambda I+ S_{X}^\top \overline{W} S_X\right)^{1/2} \left(g_\lambda\left(S_{X}^\top \overline{W} S_X\right)S_X^\top \overline{W} S_X -I\right)\left(\lambda I+ S_{X}^\top \overline{W} S_X\right)^{r-1/2}\right\|_K \\
		&~ \cdot\left\|\left(\lambda I+ S_{X}^\top \overline{W} S_X\right)^{-(r-1/2)}(\lambda I+L_K)^{r-1/2}\right\|_{op}\cdot\left\|(\lambda I+L_K)^{-(r-1/2)}L_K^{r-1/2}\right\|_{op}\cdot \left\|L_K^{1/2} u_\rho\right\|_K\\
		&=\left\|\left(\lambda I+ S_{X}^\top \overline{W} S_X\right)^{r}\left(g_\lambda\left(S_{X}^\top \overline{W} S_X\right)S_X^\top \overline{W} S_X -I\right)\right\|_{op} \cdot\left\|\left(\lambda I+ S_{X}^\top \overline{W} S_X\right)^{-(r-1/2)}(\lambda I+L_K)^{r-1/2}\right\|_{op}\\
		&~ \cdot\left\|(\lambda I+L_K)^{-(r-1/2)}L_K^{r-1/2}\right\|_{op}\cdot \left\|L_K^{1/2} u_\rho\right\|_K\\
		&\le \left\|\left(\lambda I+ S_{X}^\top \overline{W} S_X\right)^{r}\left(g_\lambda\left(S_{X}^\top \overline{W} S_X\right)S_X^\top \overline{W} S_X -I\right)\right\|_{op} \cdot\left\|\left(\lambda I+ S_{X}^\top \overline{W} S_X\right)^{-1}(\lambda I+L_K)\right\|_{op}^{r-1/2}\cdot \left\| u_\rho\right\|_{\rho_\X^{te}}\\
		&\le 2^r(b+1+\gamma_r)\lambda^r \left\|\left(\lambda I+ S_{X}^\top \overline{W} S_X\right)^{-1}(\lambda I+L_K)\right\|_{op}^{r-1/2}\cdot \left\| u_\rho\right\|_{\rho_\X^{te}},
	\end{align*}
	where the last inequality holds due to Lemma \ref{lemma: property of filter function} with $t=r$, and Lemma \ref{lemma: Cordes inequality} with $s=r-1/2$ and $A=\left(\lambda I+ S_{X}^\top \overline{W} S_X\right)^{-1}$ and $B=\lambda I+L_K.$\\
	Case 2: $r>3/2.$ \\
	In this case $r-1/2>1,$ we can rewrite the term $L_K^{r-1/2}$ as
	\begin{align*}
		L_K^{r-1/2}=L_K\cdot L_K^{r-3/2}=\left(\left(L_K- S_{X}^\top \overline{W} S_X\right)+ S_{X}^\top \overline{W} S_X\right)L_K^{r-3/2}.
	\end{align*}
	%and $\left\|L_K^{r-1/2}-\left(\lambda I+ S_{X}^\top \overline{W} S_X\right)^{r-1/2}\right\|_{op}^{r-1/2}\le  (r-1/2)\kappa^{2r-3}\left\|L_K-S_{X}^\top \overline{W} S_X\right\|.$
	Then by $\left\|L_K^{r-3/2}\right\|_{op}\le \left\|L_K\right\|_{op}^{r-3/2}\le \kappa^{2r-3}$, we have
	\begin{align*}
		I_3&\le\Big\|\left(\lambda I+ S_{X}^\top \overline{W} S_X\right)^{1/2}\left(g_\lambda\left(S_{X}^\top \overline{W} S_X\right)S_X^\top \overline{W} S_X -I\right)\left(\left(L_K- S_{X}^\top \overline{W} S_X\right)+S_{X}^\top \overline{W} S_X\right)\Big\|_{op} \\
		&~~~~ \cdot \left\|L_K^{r-3/2} \right\|_{op}\cdot \left\|L_K^{1/2} u_\rho\right\|_K\\
		&\le \left\|\left(\lambda I+ S_{X}^\top \overline{W} S_X\right)^{1/2}\left(g_\lambda\left(S_{X}^\top \overline{W} S_X\right)S_X^\top \overline{W} S_X -I\right) \left(L_K- S_{X}^\top \overline{W} S_X\right) \right\|_{op} \kappa^{2r-3}\left\| u_\rho\right\|_{\rho_\X^{te}}\\
		&+\left\|\left(\lambda I+ S_{X}^\top \overline{W} S_X\right)^{1/2}\left(g_\lambda\left(S_{X}^\top \overline{W} S_X\right)S_X^\top \overline{W} S_X -I\right) \left(\lambda I+ S_{X}^\top \overline{W} S_X\right)\right\|_{op}\kappa^{2r-3} \left\| u_\rho\right\|_{\rho_\X^{te}}\\
		&\le\left( \sqrt{2}(b+1+\gamma_{1/2})\lambda^{1/2} \left\|L_K-S_{X}^\top \overline{W} S_X\right\|_{op}
		+2^{3/2}(b+1+\gamma_{3/2})\lambda^{3/2}\right) \kappa^{2r-3}\left\| u_\rho\right\|_{\rho_\X^{te}},
	\end{align*}
	where the last inequality holds due to  Lemma \ref{lemma: property of filter function} with $t=1/2$ and $t={3/2}$.
	Then putting the estimates back into  \eqref{equation?error decomposition intermediate term} yields the desired result. We then complete the proof.
\end{proof}

Similary, to analyze the error of weighted spectral algorithm with clipped weight (\ref{algorithm: spectral algorithm with clipped weight}), we need the following error decomposition.
\begin{proposition}\label{proposition: error decompsotion for clipped weight}
	Let the spectral algorithm with clipped weight be defined by (\ref{algorithm: spectral algorithm with clipped weight}). Suppose Assumption \ref{assumption: regularity condition} holds with $1/2 \le r\le \nu_g.$ The following estimates hold:\\
	When $1/2\le r\le 3/2,$
	\begin{equation}\label{eq: error decompsotion for clipped weight r less than 3/2}
		\begin{aligned}
			\|f_{\bz,\lambda}^{\hat{\mathbf{w}}}-f_\rho\|_{\rho_\X^{te}}\le   2bJ_1^{1/2} J_2 J_3+2^r(b+1+\gamma_r)\left\| u_\rho\right\|_{\rho_\X^{te}}\lambda^r J_1^{r} J_2^{r}.
		\end{aligned}
	\end{equation}
	When $r> 3/2,$
	\begin{equation}\label{eq: error decompsotion for clipped weight  r large than 3/2}
		\begin{aligned}
			\|f_{\bz,\lambda}^{\hat{\mathbf{w}}}-f_\rho\|_{\rho_\X^{te}}&\le J_1^{1/2} J_2^{1/2}\Big(2b J_2^{1/2} J_3+ \sqrt{2}(b+1+\gamma_{1/2}) (r-1/2)\kappa^{2r-3}\\
			&\quad \cdot \left\| u_\rho\right\|_{\rho_\X^{te}}\lambda^{1/2}D_n^{r-3/2}(J_4+J_5)
			+2^r(b+1+\gamma_r)\left\| u_\rho\right\|_{\rho_\X^{te}}\lambda^r\Big),
		\end{aligned}
	\end{equation}
	where \begin{align*}
		J_1&=\left\|L_K (\lambda I+ \hat{L}_K)^{-1}\right\|_{op};\\
		J_2&= \left\|(\lambda I+ \hat{L}_K) (\lambda I+ S_{X}^\top \hat{W} S_X)^{-1}\right\|_{op};\\
		J_3&=\Big\| (\lambda I+ \hat{L}_K)^{-1/2}  (S_X^\top \hat{W}\bar{y}-S_X^\top \hat{W} S_X f_\rho)\Big\|_K;\\
		J_4&=\left\|L_K-\hat{L}_K\right\|_{op};\\
		J_5&=\left\|\hat{L}_K-S_{X}^\top \hat{W} S_X\right\|_{op},
	\end{align*}
	and	the integral operator $\hat{L}_K$ is defined as
	\begin{align*}
		\hat{L}_Kf=\int_X f(x)K(\cdot,x)\hat{w}(x)d\rho_{\X}^{tr}.
	\end{align*}
\end{proposition}	
\begin{proof}
	The proof is similar to that of Proposition \ref{proposition: error decomposition for normalization}.
	
	First, by the definition (\ref{algorithm: spectral algorithm with clipped weight}) of $f_{\bz,\lambda}^{\hat{\mathbf{w}}}$, we obtain
	\begin{align*}
		f_{\bz,\lambda}^{\hat{\mathbf{w}}}-f_\rho&=g_\lambda\left(S_{X}^\top\hat{ W} S_X\right)S_X^\top \hat{ W}\bar{y}-f_\rho\\&=g_\lambda\left(S_{X}^\top\hat{ W} S_X\right)S_X^\top \hat{ W}\bar{y}-g_\lambda\left(S_{X}^\top\hat{ W} S_X\right)S_X^\top \hat{ W} S_X f_\rho +g_\lambda\left(S_{X}^\top\hat{ W} S_X\right)S_X^\top \hat{ W} S_X f_\rho-f_\rho\\
		&=g_\lambda\left(S_{X}^\top\hat{ W} S_X\right)\left(S_X^\top \hat{ W}\bar{y}-S_X^\top \hat{ W} S_X f_\rho\right) +\left(g_\lambda\left(S_{X}^\top\hat{ W} S_X\right)S_X^\top \hat{ W} S_X -I\right)f_\rho.
	\end{align*}
	Then we see from the
	identity $\|f\|_{\rho_\X^{te}} = \left\|L_K^{1/2} f\right\|_K$ for $f\in {\mathcal H}_K$ and $f_\rho=L_K^r u_\rho $ with $u_\rho\in L_{\rho_\X^{te}}^2$ and $r\ge \frac12$, we have
	\begin{align}\label{equation: error decomposition intermediate term1}
		\|f_{\bz,\lambda}^{\hat{\mathbf{w}}}-f_\rho\|_{\rho_\X^{te}}
		&=\left\|L_K^{1/2}\left(f_{\bz,\lambda}^{\hat{\mathbf{w}}}-f_\rho\right)\right\|_K=\left\|L_K^{1/2} \left(\lambda I+ \hat{L}_K\right)^{-1/2} \left(\lambda I+ \hat{L}_K\right)^{1/2}\left(f_{\bz,\lambda}^{\hat{\mathbf{w}}}-f_\rho\right)\right\|_K \notag\\
		& \leq\left\|L_K^{1/2} \left(\lambda I+ \hat{L}_K\right)^{-1/2}\right\|_{op}
		\cdot \left\|\left(\lambda I+ \hat{L}_K\right)^{1/2} \left(\lambda I+ S_{X}^\top \hat{W} S_X\right)^{-1/2}\right\|_{op} \notag\\&~~~~\cdot \left\|\left(\lambda I+ S_{X}^\top\hat{W} S_X\right)^{1/2} \left(f_{\bz,\lambda}^{\hat{\mathbf{w}}}-f_\rho\right)\right\|_K\\
		& \le J_1^{1/2} J_2^{1/2 }\cdot\Bigg[\left\|\left(\lambda I+ S_{X}^\top \hat{W} S_X\right)^{1/2} g_\lambda\left(S_{X}^\top\hat{ W} S_X\right)\left(S_X^\top \hat{W}\bar{y}-S_X^\top \hat{W} S_X f_\rho\right)\right\|_K \notag\\
		&~~~~+ \left\|\left(\lambda I+ S_{X}^\top \hat{W} S_X\right)^{1/2}\left(g_\lambda\left(S_{X}^\top\hat{ W} S_X\right)S_X^\top \hat{W} S_X -I\right)f_\rho\right\|_K\Bigg] \notag\\
		& \le J_1^{1/2} J_2^{1/2 }\cdot\Bigg[\left\|\left(\lambda I+ S_{X}^\top \hat{W} S_X\right)^{1/2} g_\lambda\left(S_{X}^\top\hat{ W} S_X\right)\left(S_X^\top \hat{W}\bar{y}-S_X^\top \hat{W} S_X f_\rho\right)\right\|_K \notag\\
		&~~~~+ \left\|\left(\lambda I+ S_{X}^\top \hat{W} S_X\right)^{1/2}\left(g_\lambda\left(S_{X}^\top\hat{ W} S_X\right)S_X^\top \hat{W} S_X -I\right)L_K^{r-1/2}\right\|_{op}\cdot \left\| u_\rho\right\|_{\rho_\X^{te}}\Bigg], \notag
	\end{align}
	where we use Lemma \ref{lemma: Cordes inequality} with $s=1/2$ to bound $$\Big\|L_K^{1/2} \Big(\lambda I+ \hat{L}_K\Big)^{-1/2}\Big\|_{op}\le \Big\|L_K\Big(\lambda I+ \hat{L}_K\Big)^{-1}\Big\|_{op}^{1/2}=J_1^{1/2}$$ and
	$$\Big\|\Big(\lambda I+ \hat{L}_K\Big)^{1/2} \Big(\lambda I+ S_{X}^\top \hat{W} S_X\Big)^{-1/2}\Big\|_{op}\le \Big\|\Big(\lambda I+ \hat{L}_K\Big) \Big(\lambda I+ S_{X}^\top \hat{W} S_X\Big)^{-1}\Big\|_{op}^{1/2}=J_2^{1/2}.$$
	
	In the following, we further divide the term $\Bigg\|\left(\lambda I+ S_{X}^\top \hat{W} S_X\right)^{1/2} g_\lambda\left(S_{X}^\top\hat{ W} S_X\right)\Big(S_X^\top \hat{W}\bar{y}-S_X^\top \hat{W} S_X f_\rho\Big)\Bigg\|_K$ as
	\begin{align*}
		&\left\|\left(\lambda I+ S_{X}^\top \hat{W} S_X\right)^{1/2} g_\lambda\left(S_{X}^\top\hat{ W} S_X\right)\left(S_X^\top \hat{W}\bar{y}-S_X^\top \hat{W} S_X f_\rho\right)\right\|_K\\
		&=\Bigg\|\left(\lambda I+ S_{X}^\top \hat{W} S_X\right)^{1/2}g_\lambda\left(S_{X}^\top\hat{ W} S_X\right)\left(\lambda I+ S_{X}^\top \hat{W} S_X\right)^{1/2} \\
		&~~~~\left(\lambda I+ S_{X}^\top \hat{W} S_X\right)^{-1/2}\left(\lambda I+ \hat{L}_K\right)^{1/2} \left(\lambda I+ \hat{L}_K\right)^{-1/2}  \left(S_X^\top \hat{W}\bar{y}-S_X^\top \hat{W} S_X f_\rho\right)\Bigg\|_K\\
		&\le \left\|\left(\lambda I+ S_{X}^\top \hat{W} S_X\right)^{1/2} g_\lambda\left(S_{X}^\top\hat{ W} S_X\right) (\lambda I+ S_{X}^\top \hat{W} S_X)^{1/2} \right\|_{op}\\
		&~~~~\cdot\left\|\left(\lambda I+ S_{X}^\top \hat{W} S_X\right)^{-1/2} \left(\lambda I+ \hat{L}_K\right)^{1/2}\right\|_{op} \cdot\left\| \left(\lambda I+ \hat{L}_K\right)^{-1/2}  \left(S_X^\top \hat{W}\bar{y}-S_X^\top \hat{W} S_X f_\rho\right)\right\|_K\\
		&\le 2b \left\|\left(\lambda I+ S_{X}^\top \hat{W} S_X\right)^{-1} \left(\lambda I+ \hat{L}_K\right)\right\|_{op}^{1/2} \cdot\left\| \left(\lambda I+ \hat{L}_K\right)^{-1/2}  \left(S_X^\top \hat{W}\bar{y}-S_X^\top \hat{W} S_X f_\rho\right)\right\|_K\\
		&= 2bJ_2^{1/2} J_3,
	\end{align*}	
	where the last inequality follows from the property
	(\ref{condition1}) of the filter function $g_\lambda$ and the Cordes inequality with $s=1/2$ in Lemma \ref{lemma: Cordes inequality}.
	
	For the  term $	\left\|\left(\lambda I+ S_{X}^\top \hat{W} S_X\right)^{1/2}\left(g_\lambda\left(S_{X}^\top\hat{ W} S_X\right)S_X^\top \hat{W} S_X -I\right)L_K^{r-1/2}\right\|_{op}$, we consider two cases based on the regularity parameter $r$ of $f_{\rho}$. \\
	{\bf Case 1}: $1/2 \le r\le 3/2$.\\
	%In this case, $0\le r-1/2\le 1,$ we rewrite $L_K^{r-1/2}$ as
	%\begin{align*}
	%	L_K^{r-1/2}=(\lambda I+ S_{X}^\top \hat{W} S_X)^{r-1/2}(\lambda I+ S_{X}^\top \hat{W} S_X)^{-(r-1/2)}(\lambda I+L_K)^{r-1/2}(\lambda I+L_K)^{-(r-1/2)}L_K^{r-1/2}
	%\end{align*}
	In this case, $0\le r-1/2\le 1,$ we have
	\begin{align*}
		&\left\|\left(\lambda I+ S_{X}^\top \hat{W} S_X\right)^{1/2} \left(g_\lambda\left(S_{X}^\top\hat{ W} S_X\right)S_X^\top \hat{W} S_X -I\right)L_K^{r-1/2}\right\|_{op}\\
		&\le \left\|\left(\lambda I+ S_{X}^\top \hat{W} S_X\right)^{1/2} \left(g_\lambda\left(S_{X}^\top\hat{ W} S_X\right)S_X^\top \hat{W} S_X -I\right) \left(\lambda I+ S_{X}^\top \hat{W} S_X\right)^{r-1/2}\right\|_{op} \\
		&~ \cdot\left\|\left(\lambda I+ S_{X}^\top \hat{W} S_X\right)^{-(r-1/2)} \left(\lambda I+\hat{L}_K\right)^{r-1/2}\right\|_{op}\cdot\left\|\left(\lambda I+\hat{L}_K\right)^{-(r-1/2)}L_K^{r-1/2}\right\|_{op}\\
		&=\left\|\left(\lambda I+ S_{X}^\top \hat{W} S_X\right)^{r} \left(g_\lambda\left(S_{X}^\top\hat{ W} S_X\right)S_X^\top \hat{W} S_X -I\right)\right\|_{op}\\
		&~ \cdot\left\|\left(\lambda I+ S_{X}^\top \hat{W} S_X\right)^{-(r-1/2)} \left(\lambda I+\hat{L}_K\right)^{r-1/2}\right\|_{op} \cdot\left\|\left(\lambda I+\hat{L}_K\right)^{-(r-1/2)}L_K^{r-1/2}\right\|_{op}\\
		&\le 2^r(b+1+\gamma_r)\lambda^r \left\|\left(\lambda I+ S_{X}^\top \hat{W} S_X\right)^{-1} \left(\lambda I+\hat{L}_K\right)\right\|_{op}^{r-1/2} \cdot \left\|\left(\lambda I+\hat{L}_K\right)^{-1}L_K\right\|_{op}^{r-1/2}\\
		&=2^r\left(b+1+\gamma_r\right)\lambda^r J_2^{r-1/2} J_1^{r-1/2},
	\end{align*}
	where the last inequality follows from Lemma \ref{lemma: property of filter function} with $t=r$, and Lemma \ref{lemma: Cordes inequality} applied with $s=r-1/2$ to bound the operator norms $\left\|\left(\lambda I+ S_{X}^\top \hat{W} S_X\right)^{-(r-1/2)}\left(\lambda I+\hat{L}_K\right)^{r-1/2}\right\|_{op}$ and
	$\left\|\left(\lambda I+\hat{L}_K\right)^{-(r-1/2)}L_K^{r-1/2}\right\|_{op}$.\\
	{\bf Case 2:} $r>3/2.$ \\
	In this case, we see that $r-1/2>1,$ then $L_K^{r-1/2}$ can be rewritten as
	\begin{align*}
		L_K^{r-1/2}=\left(L_K^{r-1/2}- \hat{L}_K^{r-1/2}+\hat{L}_K^{r-1/2}-( S_{X}^\top \hat{W} S_X)^{r-1/2}\right)+ (S_{X}^\top \hat{W} S_X)^{r-1/2}.
	\end{align*}
	Since $\|L_K\|_{op}\le \kappa^2\le D_n\kappa^2,$ $\|\hat{L}_K\|_{op}\le  D_n\kappa^2,$ and  $\|S_{X}^\top \hat{W} S_X\|_{op}\le  D_n\kappa^2,$ by Lemma \ref{lemma: operator difference} with $t=r-1/2>1,$ we have
	\begin{align*}
		\left\|L_K^{r-1/2}- \hat{L}_K^{r-1/2}\right\|_{op}\le (r-1/2)\kappa^{2r-3} D_n^{r-3/2} \left\|L_K- \hat{L}_K\right\|_{op}= (r-1/2)\kappa^{2r-3} D_n^{r-3/2} J_4,
	\end{align*}
	and
	\begin{align*}
		\left\|\hat{L}_K^{r-1/2}-( S_{X}^\top \hat{W} S_X)^{r-1/2} \right\|_{op}&\le (r-1/2)\kappa^{2r-3} D_n^{r-3/2} \left\|\hat{L}_K-S_{X}^\top \hat{W} S_X \right\|_{op}\\
		&= (r-1/2)\kappa^{2r-3} D_n^{r-3/2} J_5.
	\end{align*}
	%and $\left\|L_K^{r-1/2}-(\lambda I+ S_{X}^\top \hat{W} S_X)^{r-1/2}\right\|_{op}^{r-1/2}\le  (r-1/2)\kappa^{2r-3}\left\|L_K-S_{X}^\top \hat{W} S_X\right\|.$
	Then it follows that
	\begin{align*}
		&\left\|\left(\lambda I+ S_{X}^\top \hat{W} S_X\right)^{1/2} \left(g_\lambda\left(S_{X}^\top\hat{ W} S_X\right)S_X^\top \hat{W} S_X -I\right)L_K^{r-1/2}\right\|_{op}\\
		&\le\Bigg\|\left(\lambda I+ S_{X}^\top \hat{W} S_X\right)^{1/2}\left(g_\lambda\left(S_{X}^\top\hat{ W} S_X\right)S_X^\top \hat{W} S_X -I\right) \Big(L_K^{r-1/2}- \hat{L}_K^{r-1/2}\\
		&~~~~+\hat{L}_K^{r-1/2}-\left( S_{X}^\top \hat{W} S_X\right)^{r-1/2}+\left( S_{X}^\top \hat{W} S_X\right)^{r-1/2}\Big) \Bigg\|_{op}\\
		&\le \left\|\left(\lambda I+ S_{X}^\top \hat{W} S_X\right)^{1/2}\left(g_\lambda\left(S_{X}^\top\hat{ W} S_X\right)S_X^\top \hat{W} S_X -I\right)\left(L_K^{r-1/2}- \hat{L}_K^{r-1/2}+\hat{L}_K^{r-1/2}-\left( S_{X}^\top \hat{W} S_X\right)^{r-1/2}\right) \right\|_{op}\\
		&~~~+\left\|\left(\lambda I+ S_{X}^\top\hat{W} S_X\right)^{1/2} \left(g_\lambda\left(S_{X}^\top\hat{ W} S_X\right)S_X^\top \hat{W} S_X -I\right) \left(\lambda I+ S_{X}^\top \hat{W} S_X\right)^{r-1/2}\right\|_{op} \\
		&\le \left\|\left(\lambda I+ S_{X}^\top \hat{W} S_X\right)^{1/2}\left(g_\lambda\left(S_{X}^\top\hat{ W} S_X\right) S_X^\top \hat{W} S_X -I\right)\right\|_{op}\cdot \Bigg(\left\|L_K^{r-1/2}- \hat{L}_K^{r-1/2}\right\|_{op}\\
		&~~~+ \left\|\hat{L}_K^{r-1/2}-( S_{X}^\top \hat{W} S_X)^{r-1/2} \right\|_{op} \Bigg)+\left\|\left(\lambda I+ S_{X}^\top\hat{W} S_X\right)^{r} \left(g_\lambda\left(S_{X}^\top\hat{ W} S_X\right)S_X^\top \hat{W} S_X -I\right) \right\|_{op} \\
		&\le \sqrt{2}\left(b+1+\gamma_{1/2}\right)\lambda^{1/2} (r-1/2)\kappa^{2r-3} D_n^{r-3/2}\left(\left\|L_K-\hat{L}_K\right\|_{op} +\left\|\hat{L}_K-S_{X}^\top \hat{W} S_X\right\|_{op} \right)\\
		&~~~+2^r\left(b+1+\gamma_r\right)\lambda^r\\
		&=\sqrt{2}\left(b+1+\gamma_{1/2}\right) (r-1/2)\kappa^{2r-3}\lambda^{1/2}D_n^{r-3/2}\left(J_4+J_5\right)
		+2^r(b+1+\gamma_r)\lambda^r,
	\end{align*}
	where the last inequality follows from Lemma \ref{lemma: property of filter function} with $t=1/2$ and $t=r$ respectively.
	%	, Lemma \ref{lemma: operator difference} applied  with $t=r-1/2$ to bound the operator norms $\left\|L_K^{r-1/2}- \hat{L}_K^{r-1/2}\right\|_{op}$ and $ \left\|\hat{L}_K^{r-1/2}-( S_{X}^\top \hat{W} S_X)^{r-1/2} \right\|_{op}$.
	Then putting the estimates back into \eqref{equation: error decomposition intermediate term1} yields the desired result.	
\end{proof}

\section{Proofs of Main Results}\label{section: proof of main results}
In this section, we will provide the proofs for the main results.
\subsection{Convergence analysis of weighted spectral algorithm with normalized weights}
This subsection presents the convergence analysis for the weighted spectral algorithm (\ref{algorithm: spectral algorithm with normalized weight}) with normalized weights. We begin by establishing several key lemmas and propositions that form the foundation of our theoretical results.
The following Bernstein's inequality for unbounded real-valued random variables can be found in \citet{Boucheron2013}.
\begin{lemma}\label{lemma: Bernstein inequaltiy for real valued random variables}
	Let $\zeta_1,\cdots,\zeta_n$ be independent identically distributed real-valued random variables, assume that there exists constants $c$, $\nu>0$, such that
	\begin{align*}
		\mathbb{E}\left[\left|\zeta_i-\mathbb{E}[\zeta_i]\right|^p\right] \le \frac12 p!{\nu}^2 c^{p-2}
	\end{align*}
	for all $p\ge2.$  Then for any $0<\delta<1$, the following bounds holds at least $1-\delta,$
	\begin{align*}
		\left\|\frac1n \sum_{i=1}^n\zeta_i -\mathbb{E}(\zeta_1)\right\|_\H\le \frac{2c\log\frac{2}{\delta}}{n}+\sqrt{\frac{2{\nu}^2\log\frac{2}{\delta}}{n}}.
	\end{align*}
\end{lemma}
Applying the above Bernstein's inequaltiy to the random variable $\zeta_i=w(x_i)$ yeilds the following proposition.
\begin{proposition}\label{proposition: expectation of the weight}
	Let the weight $w(\cdot)$ satified Assumption \ref{assumption: weight assumption}  with $0<\alpha\le 1$, with conifidence at least $1-\delta,$ we have
	\begin{equation}
		\left|\frac{1}{n}\sum_{j=1}^n w(x_j)-1\right|\le \frac{4(C+\sigma)}{\sqrt{n}}\log\frac{2}{\delta}.
	\end{equation}
\end{proposition}
\begin{proof}
	By the definition of the weight function $w(x),$ we have $\mathbb{E}_{x_i\sim\rho_{\X}^{tr}}[w(x_i)]=1$. Considering the random variable $\zeta_i=w(x_i), $ if the weight function $w(\cdot)$ satisfies Assumption \ref{assumption: weight assumption} with $0\le \alpha\le 1,$ for  all $p\in\mathbb{N}$ and $p\ge2,$ we derive the following moment bound:
	\begin{align*}
		&\mathbb{E}_{x_i\sim\rho_\X^{tr}}\left[\left|w(x_i)-\mathbb{E}_{x_i\sim\rho_\X^{tr}}[w(x_i)]\right|^p\right]\le 2^p
		\mathbb{E}_{x_i\sim\rho_\X^{tr}}[|w(x_i)|^p]=2^p\int_{\X}|w(x)|^p d\rho_\X^{tr}(x)\\
		&=2^p\int_{\X}|w(x)|^{p-1} d\rho_\X^{te}(x)\le 2^p \left(\int_{\X}|w(x)|^\frac{p-1}{\alpha} d\rho_\X^{te}(x)\right)^{\alpha}\le 2^p \frac12 p!C^{p-2} \sigma^2=\frac12 p! 4\sigma^2 (2C)^{p-2}.
	\end{align*}
	This establishes that the random variable $\zeta_i$ satisfies the conditions of Lemma \ref{lemma: Bernstein inequaltiy for real valued random variables}with parameters $c=2C$ and ${\nu}^2=4\sigma^2.$
	Applying  Lemma \ref{lemma: Bernstein inequaltiy for real valued random variables} yields that
	\begin{align*}
		\left|\frac{1}{n}\sum_{j=1}^n w(x_j)-1\right|=\left|\frac{1}{n}\sum_{j=1}^n w(x_j)-\mathbb{E}_{x_i\sim\rho_\X^{tr}}[w(x_i)]\right|\le \frac{4C\log\frac{2}{\delta}}{n}+\sqrt{\frac{8{\sigma}^2\log\frac{2}{\delta}}{n}}\le\frac{4(C+\sigma)}{\sqrt{n}}\log\frac{2}{\delta}
	\end{align*}
	holds with confidence at least $1-\delta$. This completes the proof.
\end{proof}
The following proposition, which addresses the difference between the original weights and the normalized weights, is a direct corollary of Proposition \ref{proposition: expectation of the weight}.
\begin{proposition}\label{proposition: difference of weights}
	Suppose Assumption \ref{assumption: weight assumption} holds with $0< \alpha\le 1$, then the following result holds at least $1-\delta,$
	\begin{align*}
		\left\| S_{X}^\top \overline{W} S_X-S_{X}^\top {W} S_X\right\|_{op}\le \frac{4(C+\sigma)\kappa^2}{\sqrt{n}}\log\frac{2}{\delta}.
	\end{align*}
\end{proposition}
\begin{proof}
	By the definition of the normalized weight $\bar{w}(x_i)=\frac{w(x_i)}{\frac{1}{n}\sum_{j=1}^n w(x_i)}$, $\frac{1}{n}\sum_{j=1}^n \bar{w}(x_i)=1$, and the definition of the operator norm, we have
	\begin{align*}
		\left\| S_{X}^\top \overline{W} S_X-S_{X}^\top {W} S_X\right\|_{op}&=\sup_{\|f\|_K=1}\left\|\left(S_{X}^\top \overline{W} S_X-S_{X}^\top {W} S_X \right)f\right\|_K\\
		&=\sup_{\|f\|_K=1}\left\|\frac{1}{n}\sum_{i=1}^n f(x_i)\left(\bar{w}(x_i)-w(x_i)\right)K_{x_i} \right\|_K\\
		&=\sup_{\|f\|_K=1}\left\|\frac{1}{n}\sum_{i=1}^n f(x_i)\bar{w}(x_i)\left(1-\frac{1}{n}\sum_{i=1}^n w(x_i)\right)K_{x_i} \right\|_K\\
		&\le \kappa^2\left|1-\frac{1}{n}\sum_{i=1}^n w(x_i)\right|.
	\end{align*}
	Then, under Assumption \ref{assumption: weight assumption} with $0<\alpha\le 1,$  applying Proposition \ref{proposition: expectation of the weight} yields the desired result.
\end{proof}
To prove Theorem \ref{theorem: main result with normalized weight}, we need to bound the three terms
$ \Big\|\left(\lambda I+ S_{X}^\top \overline{W} S_X\right)^{-1} (\lambda I+ L_K)\Big\|_{op}$, $\Big\| (\lambda I+ L_K)^{-1/2} \left(S_X^\top \overline{W}\bar{y}-S_X^\top \overline{W} S_X f_\rho\right)\Big\|_K$, and $\Big\|L_K-S_{X}^\top \overline{W} S_X\Big\|_{op}$ involved in the error decomposition mentioned in Proposition \ref{proposition: error decomposition for normalization}. The following Bernstein's inequality provides an upper bound for tail probability of the sum of random
variables in a Hilbert space \citep{Caponnetto07}.

\begin{lemma}\label{lemma: bernstein inequality for vector valued radom variables}
	Let $\xi_1,\cdots,\xi_n$ be a sequence of independent identically distributed random vectors on a separable Hilbert space $\H,$ assume there exists constant $\tilde{\sigma}$, $L>0$ such that
	\begin{align*}
		\mathbb{E}\|\xi_1-\mathbb{E}(\xi_1)\|_\H^p\le \frac12 p!\tilde{\sigma}^2L^{p-2}
	\end{align*}
	for all $p\ge2.$  Then for any $0<\delta<1$, the following bounds holds at least $1-\delta,$
	\begin{align*}
		\left\|\frac1n \sum_{i=1}^n\xi_i -\mathbb{E}(\xi_1)\right\|_\H\le \frac{2L\log\frac{2}{\delta}}{n}+\sqrt{\frac{2\tilde{\sigma}^2\log\frac{2}{\delta}}{n}}.
	\end{align*}
\end{lemma}
To estimate the term $ \Big\|\left(\lambda I+ S_{X}^\top \overline{W} S_X\right)^{-1} (\lambda I+ L_K)\Big\|_{op}$, we first analyze $\Big\| (\lambda I+ L_K)^{-1/2}\Big(S_{X}^\top \overline{W} S_X-L_K\Big)\Big\|_{op}$ as follows.
\begin{proposition}\label{proposition: operator}
	Suppose Assumption \ref{assumption: weight assumption} holds with $0< \alpha\le 1$, then the following result holds at least $1-2\delta,$
	\begin{align*}
		\left\| (\lambda I+ L_K)^{-1/2}\left(S_{X}^\top \overline{W} S_X-L_K\right)\right\|_{op}\le \frac{4(2C+\sigma)\kappa^2\log\frac{2}{\delta}}{n\sqrt{\lambda}}+\sqrt{\frac{8\kappa^{2+\alpha}\lambda^{-\alpha}(\mathcal{N}(\lambda))^{1-\alpha}\sigma^2\log\frac{2}{\delta}}{n}}.
	\end{align*}
\end{proposition}
\begin{proof}
	We begin our analysis by decomposing the  operator norm quantity
	$\Big\| (\lambda I+ L_K)^{-1/2}\Big(S_{X}^\top \overline{W} S_X-L_K\Big)\Big\|_{op}$ as follows \begin{align*}
		&\left\| (\lambda I+ L_K)^{-1/2}\left(S_{X}^\top \overline{W} S_X-L_K\right)\right\|_{op}=\left\| \left(\lambda I+ L_K\right)^{-1/2}\left(S_{X}^\top {W} S_X-L_K +S_{X}^\top \overline{W} S_X-S_{X}^\top {W} S_X\right)\right\|_{op}\\
		&\le \left\| \left(\lambda I+ L_K\right)^{-1/2}\left(S_{X}^\top {W} S_X-L_K\right)\right\|_{op}+\lambda^{-1/2}\left\| S_{X}^\top \overline{W} S_X-S_{X}^\top {W} S_X\right\|_{op}.
	\end{align*}
	The inequlaity holds due to the operator norm bound $\left\| \left(\lambda I+ L_K\right)^{-1/2}\right\|_{op}\le \lambda^{-1/2}.$ The bound for $\left\| S_{X}^\top \overline{W} S_X-S_{X}^\top {W} S_X\right\|_{op}$ is established in Proposition \ref{proposition: difference of weights}.
	
	Now we turn to estimate the term $\Big\| \Big(\lambda I+ L_K\Big)^{-1/2}\Big(S_{X}^\top {W} S_X-L_K\Big)\Big\|_{op}.$ We apply Lemma \ref{lemma: bernstein inequality for vector valued radom variables} to the random variable
	\begin{align*}
		\xi(x)=(\lambda I+L_K)^{-1/2} w(x) \langle K_x,\cdot\rangle K_x,\quad x\in\X.
	\end{align*}
	which takes value in $HS(\H_K)$, the Hilbert space of Hilbert-Schmidt operators on $\H_K$ with inner product $\langle A,B \rangle_{HS}={\rm Tr}(B^\top A).$ The Hilbert Schmidt norm is given by $\|A\|_{HS}=\sum_{i} \|Ae_i\|_K^2$ where $\{e_i\}$ is an orthonormal basis of $\H_K,$ and we have the norm relations $\|A\|_{op}\le \|A\|_{HS}.$ Moreover,
	\begin{align*}
		\mathbb{E}_{x\sim\rho_\X^{tr}}[\xi(x)]=\mathbb{E}_{x\sim\rho_\X^{tr}}[(\lambda I+L_K)^{-1/2} w(x) \langle K_x,\cdot\rangle K_x]=(\lambda I+L_K)^{-1/2} L_K,
	\end{align*}
	then $$(\lambda I+ L_K)^{-1/2}(S_{X}^\top W S_X-L_K)=\frac{1}{n}\sum_{i=1}^n \xi(x_i)-\mathbb{E}_{x\sim\rho_\X^{tr}}[\xi(x)].$$
	Then  for any $0<\alpha\le 1$, $p\in\mathbb{N}$ and $p\ge 2,$ we have the following moment bound for the random variable $\xi(x)$
	\begin{align}\label{equation: variance estimates}
		%\begin{split}
		&\mathbb{E}_{x\sim\rho_\X^{tr}}\left[\left\|\xi(x)-\mathbb{E}_{x\sim\rho_\X^{tr}}[\xi(x)]\right\|_{HS}^p\right]\le 2^p \mathbb{E}_{x\sim\rho_\X^{tr}}\left[\left\|\xi(x)\right\|_{HS}^p\right] \nonumber\\
		&= 2^p\int_\X \left\|(\lambda I+L_K)^{-1/2} w(x) \langle K_x,\cdot\rangle K_x \right\|_{HS}^p d\rho_\X^{tr}(x)\nonumber\\
		&= 2^p\int_\X \left\|(\lambda I+L_K)^{-1/2}  \langle K_x,\cdot\rangle K_x \right\|_{HS}^p (w(x))^{p-1} d\rho_\X^{te}(x)\nonumber\\
		&= 2^p\int_\X \left\|(\lambda I+L_K)^{-1/2}  \langle K_x,\cdot\rangle K_x \right\|_{HS}^{p-2+2\alpha} \nonumber\\
		&\quad \quad \cdot\left\|(\lambda I+L_K)^{-1/2}  \langle K_x,\cdot\rangle K_x \right\|_{HS}^{2-2\alpha} (w(x))^{p-1} d\rho_\X^{te}(x)\\
		&\le  2^p (\kappa^2 \lambda^{-1/2})^{p-2+2\alpha}\int_\X  \left\|(\lambda I+L_K)^{-1/2}  \langle K_x,\cdot\rangle K_x \right\|_{HS}^{2-2\alpha} (w(x))^{p-1} d\rho_\X^{te}(x)\nonumber\\
		& \le 2^p (\kappa^2 \lambda^{-1/2})^{p-2+2\alpha}\left(\int_\X  \left\|(\lambda I+L_K)^{-1/2}  \langle K_x,\cdot\rangle K_x \right\|_{HS}^{(2-2\alpha)\cdot\frac{1}{1-\alpha}}d\rho_\X^{te}(x)\right)^{1-\alpha}\nonumber\\
		&\quad \quad \cdot \left(\int_\X (w(x))^\frac{p-1}{\alpha} d\rho_\X^{te}(x)\right)^{\alpha}\nonumber\\
		&=  2^p (\kappa^2 \lambda^{-1/2})^{p-2+2\alpha}\left(\int_\X  \left\|(\lambda I+L_K)^{-1/2}  \langle K_x,\cdot\rangle K_x \right\|_{HS}^{2}d\rho_\X^{te}(x)\right)^{1-\alpha}\nonumber\\
		&\qquad \cdot \left(\int_\X (w(x))^\frac{p-1}{\alpha} d\rho_\X^{te}(x)\right)^{\alpha}.\nonumber
		%\end{split}
	\end{align}
	Here we use the the bound $\left\|(\lambda I+L_K)^{-1/2}  \langle K_x,\cdot\rangle K_x \right\|_{HS}\le \kappa\lambda^{-1/2}$ and Cauchy-Schwarz inequality.
	Next we further estimate $\int_\X  \left\|(\lambda I+L_K)^{-1/2}  \langle K_x,\cdot\rangle K_x \right\|_{HS}^{2}d\rho_\X^{te}(x).$ Let $\{(\lambda_i,
	\phi_i)\}_{i}$ be a set of normalized eigenpairs of $L_K$ on
	$\mathcal H_K$ with $\{\phi_i\}_{i=1}^\infty$ forming an
	orthonormal basis of $\mathcal H_K$, then by the Mercer Theorem, we have
	\begin{align*}
		K(x,x')=\sum_{i=1}^\infty \phi_i(x)\phi_i(x'), \quad \forall x,x'\in\X.
	\end{align*}
	Moreover, by the reproducing property, we have  $\langle K_x,\cdot\rangle K_x \phi_i=\phi_i(x)K_x$ and $K_x=\sum_{\ell=1}^\infty\langle K_x,\phi_\ell\rangle \phi_\ell =\sum_{\ell=1}^\infty\phi_\ell(x)\phi_\ell.$ Then the definition of the Hilbert-Schmidt (HS) norm implies that
	\begin{align*}
		&\left\|(\lambda I+L_K)^{-1/2}  \langle K_x,\cdot\rangle K_x\right\|_{HS}^2=\sum_{i=1}^\infty\left\|(\lambda I+L_K)^{-1/2}  \langle K_x,\cdot\rangle K_x \phi_i(x)\right\|_K^2\\
		&=\sum_{i=1}^\infty\left\|(\lambda I+L_K)^{-1/2} \phi_i(x)\sum_{\ell=1}^\infty\phi_\ell(x)\phi_\ell\right\|_K^2
		=\sum_{i=1}^\infty (\phi_i(x))^2\left\| \sum_{\ell=1}^\infty\phi_\ell(x) \frac{1}{\sqrt{\lambda+\lambda_\ell}}\phi_\ell\right\|_K^2\\
		&=\sum_{i=1}^\infty (\phi_i(x))^2 \sum_{\ell=1}^\infty \frac{(\phi_\ell(x))^2}{\lambda+\lambda_\ell}
		=K(x,x) \sum_{\ell=1}^\infty \frac{(\phi_\ell(x))^2}{\lambda+\lambda_\ell}\le \kappa^2\sum_{\ell=1}^\infty \frac{(\phi_\ell(x))^2}{\lambda+\lambda_\ell}.
	\end{align*}
	Therefore,
	\begin{align*}
		\mathbb{E}_{x\sim \rho_\X^{te}}\left[\left\|(\lambda I+L_K)^{-1/2}  \langle K_x,\cdot\rangle K_x\right\|_{HS}^2\right]\le \kappa^2\int_\X  \sum_{\ell=1}^\infty \frac{(\phi_\ell(x))^2}{\lambda+\lambda_\ell} d\rho_\X^{te}(x)
		= \kappa^2 \mathcal{N}(\lambda),
	\end{align*}
	here we use the fact $\int_X (\phi_\ell(x))^2 d\rho_\X^{te}(x)=\left\|\sqrt{\lambda_\ell} \frac{\phi_\ell}{\sqrt{\lambda_l}}\right\|_{\rho_\X^{te}}^2=\lambda_\ell.$ Then, under  Assumption \ref{assumption: weight assumption} with $0< \alpha \leq 1$, we can substitute the above estimates back into (\ref{equation: variance estimates}), resulting in the following expression for any $p\in\mathbb{N}$ and $p\ge 2,$
	\begin{align*}
		\mathbb{E}_{x\sim\rho_\X^{tr}}\left[\left\|\xi(x)-\mathbb{E}_{x\sim\rho_\X^{tr}}[\xi(x)]\right\|_{HS}^p\right]&\le 2^p (\kappa^2 \lambda^{-1/2})^{p-2+2\alpha}  (\kappa^2 \mathcal{N}(\lambda))^{1-\alpha} \cdot\frac12 p!C^{p-2} \sigma^2\\
		&=\frac12 p! (2C\kappa^2 \lambda^{-1/2})^{p-2} (4\kappa^{2+\alpha}\lambda^{-\alpha}(\mathcal{N}(\lambda))^{1-\alpha}\sigma^2).
	\end{align*}
	Applying  Lemma \ref{lemma: bernstein inequality for vector valued radom variables}  to the random variable $\xi(x)=(\lambda I+L_K)^{-1/2} w(x) \langle K_x,\cdot\rangle K_x,$ with $L=2C\kappa^2 \lambda^{-1/2}$, and $\tilde{\sigma}^2=4\kappa^{2+\alpha}\lambda^{-\alpha}(\mathcal{N}(\lambda))^{1-\alpha}\sigma^2$, then for any $0<\delta<1$, with confidence at least $1-\delta,$ there holds
	\begin{align*}
		\left\|(\lambda I+ L_K)^{-1/2}(S_{X}^\top W S_X-L_K)\right\|_{op}&\le	\left\|(\lambda I+ L_K)^{-1/2}(S_{X}^\top W S_X-L_K)\right\|_{HS}\\
		%		&\le \left\|\frac{1}{n}\sum_{i=1}^n \xi(x_i)-\mathbb{E}_{x\sim\rho_\X^{tr}}[\xi(x)]\right\|_K\\
		&\le  \frac{4C\kappa^2\log\frac{2}{\delta}}{n\sqrt{\lambda}}+\sqrt{\frac{8\kappa^{2+\alpha}\lambda^{-\alpha}(\mathcal{N}(\lambda))^{1-\alpha}\sigma^2\log\frac{2}{\delta}}{n}}.
	\end{align*}
	Combining this bound with the bound for $\left\| S_{X}^\top \overline{W} S_X-S_{X}^\top {W} S_X\right\|_{op}$ established in Proposition \ref{proposition: difference of weights}, we obtain the desired estimate and complete the proof.
\end{proof}
By Proposition \ref{proposition: operator} and the second order decomposition of inverse operator differences established in \citet{lin17,guolinzhou2017}, we obtain the following bound for $\left\|\left(\lambda I+ S_{X}^\top \overline{W} S_X\right)^{-1} (\lambda I+ L_K)\right\|_{op}.$
\begin{proposition}\label{proposition: I1}
	Suppose Assumption \ref{assumption: weight assumption} holds with $0< \alpha\le 1$, then the following result holds with confidence at least $1-\delta,$
	\begin{align*}
		&\left\|\left(\lambda I+ S_{X}^\top \overline{W} S_X\right)^{-1} (\lambda I+ L_K)\right\|_{op} \\
		& \le \left(\frac{4(C+\sigma)\kappa^2}{\sqrt{n}{\lambda}}+\sqrt{\frac{8\kappa^{2+\alpha}\lambda^{-1-\alpha}(\mathcal{N}(\lambda))^{1-\alpha}\sigma^2}{n}}+1\right)^2 \left(\log\frac{4}{\delta}\right)^2.
	\end{align*}
\end{proposition}
\begin{proof}
	Let $A$ and $B$ be invertible operators on a Banach space. By the second order decomposition of inverse operator differences proposed in \citet{lin17}, we have
	\begin{align*}
		A^{-1}-B^{-1}&=A^{-1}(B-A)B^{-1}=(A^{-1}-B^{-1})(B-A)B^{-1}+B^{-1}(B-A)B^{-1}\\
		&=A^{-1}(B-A)B^{-1}(B-A)B^{-1}+B^{-1}(B-A)B^{-1},
	\end{align*}
	then
	\begin{equation}\label{eq: A^{-1}B}
		\begin{aligned}
			A^{-1}B&=(A^{-1}-B^{-1}+B^{-1})B=(A^{-1}-B^{-1})B+I\\
			&=A^{-1}(B-A)B^{-1}(B-A) +B^{-1}(B-A)+I.
		\end{aligned}
	\end{equation}
	Using $\left\|\left(\lambda I+ S_{X}^\top \overline{W} S_X\right)^{-1}\right\|_{op}\le \lambda^{-1}$, $\left\|(\lambda I+ L_K)^{-1/2}\right\|_{op}\le \lambda^{-\frac12}$ and taking $A=\lambda I+ S_{X}^\top \overline{W} S_X$ and  $B=\lambda I+ L_K$ in \eqref{eq: A^{-1}B} yields
	\begin{align*}
		&\left\|\left(\lambda I+ S_{X}^\top \overline{W} S_X\right)^{-1} (\lambda I+ L_K)\right\|_{op}\\
		&= \Big\|\left(\lambda I+ S_{X}^\top \overline{W} S_X\right)^{-1}\left(L_K-S_{X}^\top \overline{W} S_X\right) (\lambda I+ L_K)^{-1}\left(L_K-S_{X}^\top \overline{W} S_X\right)\\
		&~~~~+(\lambda I+ L_K)^{-1}\left(L_K-S_{X}^\top \overline{W} S_X\right) +I\Big\|_{op}\\
		&\le  \lambda^{-1}\left\|\left(L_K-S_{X}^\top \overline{W} S_X\right) (\lambda I+ L_K)^{-1/2}\right\|\cdot \left\| (\lambda I+ L_K)^{-1/2}\left(L_K-S_{X}^\top \overline{W} S_X\right)\right\|_{op} \\
		&~~~~+\left\|\left(L_K-S_{X}^\top \overline{W} S_X\right) (\lambda I+ L_K)^{-1/2}\right\|_{op}\cdot\lambda^{-1/2}+1\\
		%		&=1+ \left\| (\lambda I+ L_K)^{-1/2}(S_{X}^\top W S_X-L_K)\right\|^2 \lambda^{-1}+\left\| (\lambda I+ L_K)^{-1/2}(S_{X}^\top W S_X-L_K)\right\|\cdot\lambda^{-1/2}\\
		&\le \left(\left\| (\lambda I+ L_K)^{-1/2}\left(L_K-S_{X}^\top \overline{W} S_X\right)\right\|_{op}\cdot\lambda^{-1/2}+1\right)^2,
	\end{align*}
	where the last inequality holds due to the fact that
	$\Big\|L_1L_2\Big\|_{op}=\Big\|(L_1L_2)^T\Big\|_{op}=\Big\|L_2^TL_1^T\Big\|_{op}=\Big\|L_2L_1\Big\|_{op}$
	for any self-adjoint operators $L_1$, $L_2$ on Hilbert spaces.
	Applying Proposition \ref{proposition: operator}, we get the following result with confidence at least $1-2\delta$
	\begin{align*}
		&\left\|\left(\lambda I+ S_{X}^\top \overline{W} S_X\right)^{-1} (\lambda I+ L_K)\right\|_{op}\\
		&\le \left(\left\| \left(\lambda I+ L_K\right)^{-1/2}\left(L_K-S_{X}^\top \overline{W} S_X\right)\right\|_{op}\cdot \lambda^{-1/2}+1\right)^2\\
		&\le \left(\frac{4(C+\sigma)\kappa^2}{\sqrt{n}{\lambda}}\log\frac{2}{\delta}+\sqrt{\frac{8\kappa^{2+\alpha}\lambda^{-1-\alpha}(\mathcal{N}(\lambda))^{1-\alpha}\sigma^2\log\frac{2}{\delta}}{n}}+1\right)^2.
	\end{align*}
	Then we finish our proof by scaling $2\delta$ to $\delta$.
	%		, for any $\delta>0,$
	%		with confidence at least $1-2\delta,$ there holds
	%		\begin{align*}
		%			\left\|(\lambda I+ S_{X}^\top W S_X)^{-1} (\lambda I+ L_K)\right\|_{op}
		%			\le \Big(\frac{4C\kappa^2\log\frac{2}{\delta}}{n{\lambda}}+\sqrt{\frac{8\kappa^{2+\alpha}\lambda^{-1-\alpha}(\mathcal{N}(\lambda))^{1-\alpha}\sigma^2\log\frac{2}{\delta}}{n}}+\frac{4(C+\sigma)\kappa^2}{\sqrt{n}\lambda}\log\frac{2}{\delta}+1\Big)^2.
		%		\end{align*}
	%		This completes the proof.
\end{proof}

Now we analyze the second term $\left\|(\lambda I+ L_K)^{-1/2} \left(S_X^\top \overline{W}\bar{y}-S_X^\top \overline{W} S_X f_\rho\right)\right\|_{K}$ in Proposition \ref{proposition: error decomposition for normalization}.
\begin{proposition}\label{proposition: frho}
	Under Assumption \ref{assumption: weight assumption} with $0< \alpha\le 1,$ for any $\delta>0,$
	with confidence at least $1-\delta,$ there holds
	\begin{align*}
		&\left\|(\lambda I+ L_K)^{-1/2} \left(S_X^\top \overline{W}\bar{y}-S_X^\top \overline{W} S_X f_\rho\right)\right\|_{K}\nonumber\\
		&\quad \quad \le \left(\sqrt{\frac{8M^2\kappa^{2\alpha}\lambda^{-\alpha}(\mathcal{N}(\lambda))^{1-\alpha}\sigma^2}{n}}+\frac{4M\kappa(3C+2\sigma)}{\sqrt{n\lambda}}\right)\log\frac{4}{\delta}.
	\end{align*}
\end{proposition}
\begin{proof}
	First, we divide the term $S_X^\top \overline{W}\bar{y}-S_X^\top \overline{W} S_X f_\rho$ into several parts as follows,
	\begin{align*}
		&S_X^\top \overline{W}\bar{y}-S_X^\top \overline{W} S_X f_\rho=S_X^\top {W}\bar{y}-S_X^\top {W} S_X f_\rho + S_X^\top (\overline{W}-W)\bar{y}+S_X^\top \left(W-\overline{W} \right)S_X f_\rho\\
		&=S_X^\top {W}\bar{y}-S_X^\top {W} S_X f_\rho + \frac{1}{m}\sum_{i=1}^m  \left(\bar{w}(x_i)-w(x_i)\right)y_i K_{x_i}+\frac{1}{m}\sum_{i=1}^m  \left({w}(x_i)-\bar{w}(x_i)\right)f_\rho(x_i) K_{x_i}\\
		&= S_X^\top {W}\bar{y}-S_X^\top {W} S_X f_\rho + \frac{1}{m}\sum_{i=1}^m \bar{w}(x_i) \left(1-\frac{1}{m}\sum_{j=1}^mw(x_j)\right)y_i K_{x_i}\\
		&\qquad+\frac{1}{m}\sum_{i=1}^m  \bar{w}(x_i) \left(\frac{1}{m}\sum_{j=1}^m w(x_j)-1\right)f_\rho(x_i) K_{x_i}.
	\end{align*}
	Then we have the following error decomposition.
	\begin{align*}
		&\left\|(\lambda I+ L_K)^{-1/2} \left(S_X^\top \overline{W}\bar{y}-S_X^\top \overline{W} S_X f_\rho\right)\right\|_{K}\\
		&\le \left\|(\lambda I+ L_K)^{-1/2} \left(S_X^\top {W}\bar{y}-S_X^\top {W} S_X f_\rho\right)\right\|_{K}+2M\kappa\lambda^{-1/2}\left|1-\frac{1}{m}\sum_{j=1}^mw(x_j)\right|.
	\end{align*}
	The second term $\left|1-\frac{1}{m}\sum_{j=1}^mw(x_j)\right|$ is bounded in Proposition \ref{proposition: expectation of the weight}. Next, we estimate the remaining term $\left\|(\lambda I+ L_K)^{-1/2} \left(S_X^\top {W}\bar{y}-S_X^\top {W} S_X f_\rho\right)\right\|_{K}$.
	We consider the random variable $\eta(x,y)=(\lambda I+L_K)^{-1/2} w(x) (y-f_\rho(x)) K_x,$ which takes value in $\H_K.$ One can easily see that  $\mathbb{E}_{(x,y)\sim\rho^{tr}}[\eta(x,y)]=0.$
	Then we have
	$$(\lambda I+ L_K)^{-1/2} \left(S_X^\top W\bar{y}-S_X^\top W S_X f_\rho\right)=\frac{1}{n}\sum_{i=1}^n \eta(x_i,y_i)-\mathbb{E}_{(x,y)\sim\rho^{tr}}[\eta(x,y)].$$
	It is easy to see that	
	\begin{align*}
		&\int_\X  \left\|(\lambda I+L_K)^{-1/2}   K_x \right\|_K^{2}d\rho_\X^{te}(x)= \int_\X  {\rm Tr}\left(\left(\lambda I+L_K\right)^{-1/2}K_x\otimes (\lambda I+L_K)^{-1/2}   K_x\right)d\rho_\X^{te}(x)\\
		&= \int_\X  {\rm Tr}\left((\lambda I+L_K)^{-1}K_x\otimes  K_x\right)d\rho_\X^{te}(x)
		=   {\rm Tr}\left(\int_\X(\lambda I+L_K)^{-1}K_x\otimes  K_x d\rho_\X^{te}(x)\right)	=\mathcal{N}(\lambda).
	\end{align*}
	Then under Assumption \ref{assumption: weight assumption} with $0< \alpha\le 1,$ and by Cauchy-Schwarz inequality, for any $p\in \mathbb{N}$ and $p\ge 2,$ we have the following moment bound for $\eta(x,y)$
	\begin{align*}
		&\mathbb{E}_{(x,y)\sim\rho^{tr}}\left[\left\|\eta(x,y)-\mathbb{E}_{(x,y)\sim\rho^{tr}}[\eta(x,y)]\right\|_K^p\right]= \mathbb{E}_{(x,y)\sim\rho^{tr}}\left[\left\|\eta(x,y)\right\|_K^p\right]\\
		&= \int_\X \left\|(\lambda I+L_K)^{-1/2} w(x) (y-f_\rho(x)) K_x \right\|_K^p d\rho^{tr}(x,y)\\
		&= \int_\X \left\|(\lambda I+L_K)^{-1/2}  K_x \right\|_K^p |y-f_\rho(x)|^p (w(x))^{p} d\rho^{tr}(x,y)\\
		&\le \int_\X \left\|(\lambda I+L_K)^{-1/2}  K_x \right\|_K^p (2M)^p (w(x))^{p-1} d\rho_\X^{te}(x)\\
		&\le  (2M)^p(\kappa \lambda^{-1/2})^{p-2+2\alpha}\left(\int_\X  \left\|(\lambda I+L_K)^{-1/2}  K_x \right\|_K^{(2-2\alpha)\cdot\frac{1}{1-\alpha}}d\rho_\X^{te}(x)\right)^{1-\alpha} \left( (w(x))^\frac{p-1}{\alpha} d\rho_\X^{te}(x)\right)^{\alpha}\\
		&= (2M)^p (\kappa \lambda^{-1/2})^{p-2+2\alpha}\left(\int_\X  \left\|(\lambda I+L_K)^{-1/2}   K_x \right\|_K^{2}d\rho_\X^{te}(x)\right)^{1-\alpha} \left( (w(x))^\frac{p-1}{\alpha} d\rho_\X^{te}(x)\right)^{\alpha}\\
		&\le(2M)^p (\kappa \lambda^{-1/2})^{p-2+2\alpha}  ( \mathcal{N}(\lambda))^{1-\alpha} \cdot\frac12 p!C^{p-2} \sigma^2\\
		&=\frac12 p! (2MC\kappa\lambda^{-1/2})^{p-2}(4M^2 \kappa^{2\alpha}\lambda^{-\alpha} (\mathcal{N}(\lambda))^{1-\alpha}\sigma^2).
	\end{align*}
	Applying  Lemma \ref{lemma: bernstein inequality for vector valued radom variables}  to the random variable $\eta(x,y)=(\lambda I+L_K)^{-1/2} w(x) (y-f_\rho(x)) K_x$  with $L=2MC\kappa \lambda^{-1/2}$ and $\tilde{\sigma}^2=4M^2\kappa^{2\alpha}\lambda^{-\alpha}(\mathcal{N}(\lambda))^{1-\alpha}\sigma^2$, we have with confidence at least $1-\delta,$
	\begin{align*}
		\left\|(\lambda I+ L_K)^{-1/2} (S_X^\top {W}\bar{y}-S_X^\top {W} S_X f_\rho)\right\|_{K}&=
		\left\|\frac{1}{n}\sum_{i=1}^n \eta(x_i,y_i)-\mathbb{E}_{(x,y)\sim\rho^{tr}}[\eta(x,y)]\right\|_K\\
		&
		\le  \frac{4MC\kappa\log\frac{2}{\delta}}{n\sqrt{\lambda}}+\sqrt{\frac{8M^2\kappa^{2\alpha}\lambda^{-\alpha}(\mathcal{N}(\lambda))^{1-\alpha}\sigma^2\log\frac{2}{\delta}}{n}}.
	\end{align*}
	Combining this with the bound for $\left|1 - \frac{1}{m}\sum_{j=1}^m w(x_j)\right|$ from Proposition \ref{proposition: expectation of the weight}, and scaling $2\delta$ to $\delta$, we conclude the proof of the proposition.
\end{proof}
In the following, we will estimate the third term $\left\|S_{X}^\top \overline{W} S_X-L_K\right\|_{op}$ in Propostion \ref{proposition: error decomposition for normalization}.
\begin{proposition}\label{proposition: operator2} For any $0<\delta<1,$
	with confidence at least $1-\delta,$ there holds
	\begin{align*}
		\left\|S_{X}^\top \overline{W} S_X-L_K\right\|_{op} \le \frac{8\kappa^2(C+\sigma)}{\sqrt{n}}\log\frac{4}{\delta}.
	\end{align*}
\end{proposition}
\begin{proof}
	First, we have
	\begin{align*}
		\left\|S_{X}^\top \overline{W} S_X-L_K\right\|_{op} &=\left\|S_{X}^\top \overline{W} S_X-S_{X}^\top {W} S_X+S_{X}^\top {W} S_X-L_K\right\|_{op}\\&\le \left\|S_{X}^\top \overline{W} S_X-S_{X}^\top {W} S_X \right\|_{op}+\left\|S_{X}^\top {W} S_X-L_K\right\|_{op}.
	\end{align*}
	The term $\left\|S_{X}^\top \overline{W} S_X-S_{X}^\top {W} S_X \right\|_{op}$ is bounded in Proposition \ref{proposition: difference of weights}. Additionally, the term  $\left\|S_{X}^\top {W} S_X-L_K\right\|_{op}$ can be bounded using techniques similar to those applied to $\Big\|(\lambda I+L_K)^{-1/2}\Big(S_{X}^\top {W} S_X-L_K\Big)\Big\|_{op}$ in Proposition \ref{proposition: operator}.
	We consider the random variable $\zeta(x)=w(x)\langle \cdot, K_x\rangle K_x$ , which takes value in $HS(\H_K)$. And one can easily see that $\mathbb{E}_{x\sim \rho_\X^{tr}}[\zeta(x)]=L_K$, then $$S_{X}^\top W S_X-L_K=\frac{1}{n}\sum_{i=1}^n \zeta(x_i)-\mathbb{E}_{x\sim \rho_\X^{tr}}[\zeta(x)].$$
	Moreover, since
	\begin{align*}
		&\left\| \langle K_x,\cdot\rangle K_x\right\|_{HS}^2=\sum_{i=1}^\infty\left\|  \langle K_x,\cdot\rangle K_x \phi(x)\right\|_K^2
		=\sum_{i=1}^\infty\left\| \phi_i(x)\sum_{\ell=1}^\infty\phi_\ell(x)\phi_\ell\right\|_K^2\\
		&=\sum_{i=1}^\infty (\phi_i(x))^2\left\| \sum_{\ell=1}^\infty\phi_\ell(x)\phi_\ell\right\|_K^2
		=\sum_{i=1}^\infty (\phi_i(x))^2 \sum_{\ell=1}^\infty (\phi_\ell(x))^2
		=(K(x,x))^2\le \kappa^4.
	\end{align*}
	Then for any $p\in\mathbb{N}$ and $p\ge 2,$
	\begin{align*}
		&\mathbb{E}_{x\sim\rho_\X^{tr}}\left[\left\|\zeta(x)-\mathbb{E}_{x\sim\rho_\X^{tr}}[\zeta(x)]\right\|_{HS}^p\right]\le 2^p \mathbb{E}_{x\sim\rho_\X^{tr}}\left[\left\|\zeta(x)\right\|_{HS}^p\right]= 2^p\int_\X \left\|w(x) \langle K_x,\cdot\rangle K_x \right\|_{HS}^p d\rho_\X^{tr}(x)\\
		&= 2^p\int_\X \left\| \langle K_x,\cdot\rangle K_x \right\|_{HS}^p (w(x))^{p-1} d\rho_\X^{te}(x)
		\le 2^p\int_\X \kappa^{2p}  (w(x))^{p-1} d\rho_\X^{te}(x)\\
		&\le 2^p \kappa^{2p} \left( (w(x))^\frac{p-1}{\alpha} d\rho_\X^{te}(x)\right)^{\alpha}
		\le 2^p \kappa^{2p}\cdot\frac12 p!C^{p-2} \sigma^2
		=\frac12 p! (2C\kappa^2 )^{p-2} (4\kappa^4\sigma^2).
	\end{align*}
	Applying Lemma \ref{lemma: bernstein inequality for vector valued radom variables}  to the random variable $\zeta(x)=w(x)\langle \cdot, K_x\rangle K_x$ with $L=2C\kappa^2 $ and $\tilde{\sigma}^2=4\kappa^4\sigma^2,$ we obtain, with confidence at least $1-\delta,$
	\begin{align*}
		\left\|S_{X}^\top W S_X-L_K\right\|_{op}\le \left\|S_{X}^\top W S_X-L_K\right\|_{HS}\le \frac{4C\kappa^2\log\frac{2}{\delta}}{n}+\sqrt{\frac{8\kappa^4\sigma^2\log\frac{2}{\delta}}{n}}\le \frac{4\kappa^2(C+\sigma)}{\sqrt{n}}\log\frac{2}{\delta}.
	\end{align*}
	The proof is completed by combining the derived bound with the bound for $\Big\|S_{X}^\top \overline{W} S_X-S_{X}^\top {W} S_X \Big\|_{op}$ from Proposition \ref{proposition: difference of weights} by scaling $2\delta$ to $\delta$.
\end{proof}
Now we are ready to prove our main result for the spectral algorithms with normalized weight $\bar{w}.$

\noindent{\bf Proof of Theorem \ref{theorem: main result with normalized weight}.} We obtain the  desired upper bounds by combining the bounds from Proposition \ref{proposition: I1}, Proposition \ref{proposition: frho} and Proposition \ref{proposition: operator2} with appropriate parameters. Under Assumption \ref{assumption: effective dimension} with $0<\beta\le 1$, Assumption \ref{assumption: regularity condition} with $1/2\le r\le \nu_g$, Assumption \ref{assumption: weight assumption} with $0<\alpha\le 1$, let the regularization parameter be set as $\lambda=n^{-\frac{1}{\min\{2r,3\}+\beta+\alpha(1-\beta)}}$ with  $\beta+\alpha(1-\beta)\ge1$.

First,	by Proposition \ref{proposition: I1}, for $\delta\in(0,1),$
there exists a subset $\mathcal Z_{\delta,1}^{|D|}$ of $\mathcal
Z^{|D|}$ of measure at least $1-\delta$ such that
\begin{align*}
	&\left\|\left(\lambda I+ S_{X}^\top \overline{W} S_X\right)^{-1} (\lambda I+ L_K)\right\|_{op} \\
	&\le \left(\sqrt{\frac{8\kappa^{2+\alpha}\lambda^{-1-\alpha}(\mathcal{N}(\lambda))^{1-\alpha}\sigma^2}{n}}+\frac{4(2C+\sigma)\kappa^2}{\sqrt{n}\lambda}+1\right)^2\left(\log\frac{4}{\delta}\right)^2\\
	&\le \left(\sqrt{8\kappa^{2+\alpha}(C_0)^{1-\alpha}\sigma^2}+4(2C+\sigma)\kappa^2+1\right)^2 \left(\log\frac{4}{\delta}\right)^2.
\end{align*}
According to Proposition \ref{proposition: frho},
there exists another subset $\mathcal Z_{\delta,2}^{|D|}$ of $\mathcal Z^{|D|}$ of measure
at least $1-\delta$ such that
\begin{align*}
	&\left\|(\lambda I+ L_K)^{-1/2} \left(S_X^\top \overline{W}\bar{y}-S_X^\top \overline{W} S_X f_\rho\right)\right\|_{K}\nonumber\\
	& \le \left(\sqrt{\frac{8M^2\kappa^{2\alpha}\lambda^{-\alpha}(\mathcal{N}(\lambda))^{1-\alpha}\sigma^2}{n}}+\frac{4M\kappa(3C+2\sigma)}{\sqrt{n\lambda}}\right)\log\frac{4}{\delta}\\
	&\le \left(\sqrt{{8M^2\kappa^{2\alpha}(C_0)^{1-\alpha}\sigma^2}}+4M\kappa(3C+2\sigma)\right) n^{-\frac{\min\{r,3/2\}}{\min\{2r,3\}+\beta+\alpha(1-\beta)}}\log\frac{4}{\delta}.
\end{align*}
When $\frac12\le r\le \frac{3}{2}$, $\lambda=n^{-\frac{1}{2r+\beta+\alpha(1-\beta)}}$, putting the above results back into Proposition \ref{proposition: error decomposition for normalization}, and $ D\in\mathcal Z_{\delta,1}^{|D|}\bigcap
\mathcal Z_{\delta,2}^{|D|},$
the following inequality holds with confidence at least $1-2\delta$,
\begin{align*}
	\|f_{\bz,\lambda}^{\overline{\mathbf{w}}}-f_\rho\|_{\rho_\X^{te}}&\le
	2b \left\|\left(\lambda I+ S_{X}^\top \overline{W} S_X\right)^{-1} \left(\lambda I+ L_K\right)\right\|_{op}\cdot \left\| \left(\lambda I+ L_K\right)^{-1/2} \left(S_X^\top \overline{W}\bar{y}-S_X^\top \overline{W} S_X f_\rho\right)\right\|_K\\
	&+ C_r\lambda^r \left\|\left(\lambda I+ S_{X}^\top \overline{W} S_X\right)^{-1}(\lambda I+L_K)\right\|_{op}^{r}.\\
	&\le C_1 n^{-\frac{r}{2r+\beta+\alpha(1-\beta)}}  \left(\log\frac{4}{\delta}\right)^3.
\end{align*}
where
\begin{align*}
	C_1&=2b\left(\sqrt{8\kappa^{2+\alpha}(C_0)^{1-\alpha}\sigma^2}+4(2C+\sigma)\kappa^2+1\right)^2\left(\sqrt{{8M^2\kappa^{2\alpha}(C_0)^{1-\alpha}\sigma^2}}+4M\kappa(3C+2\sigma)\right)\\
	&\qquad+C_r\left(\sqrt{8\kappa^{2+\alpha}(C_0)^{1-\alpha}\sigma^2}+4(2C+\sigma)\kappa^2+1\right)^{2r}.
\end{align*}
Moreover, by Proposition \ref{proposition: operator2}
there exists another subset $\mathcal Z_{\delta,3}^{|D|}$ of $\mathcal Z^{|D|}$ of measure
at least $1-\delta$ such that
\begin{align*}
	\left\|S_{X}^\top \overline{W} S_X-L_K\right\|_{op} \le\frac{8\kappa^2(C+\sigma)}{\sqrt{n}}\log\frac{4}{\delta}.
\end{align*}
Therefore,  when $r>\frac{3}{2},$ $\lambda=n^{-\frac{1}{3+\beta+\alpha(1-\beta)}}$ and $ D\in\mathcal Z_{\delta,1}^{|D|}\bigcap
\mathcal Z_{\delta,2}^{|D|}\bigcap \mathcal Z_{\delta,3}^{|D|},$
with confidence at least $1-3\delta,$ there holds
\begin{align*}
	\|f_{\bz,\lambda}^{\overline{\mathbf{w}}}-f_\rho\|_{\rho_\X^{te}}&\le
	2b \left\|\left(\lambda I+ S_{X}^\top \overline{W} S_X\right)^{-1} \left(\lambda I+ L_K\right)\right\|_{op}\cdot \left\| (\lambda I+ L_K)^{-1/2} \left(S_X^\top \overline{W}\bar{y}-S_X^\top \overline{W} S_X f_\rho\right)\right\|_K\\
	&+ C_r \left\|\left(\lambda I+ S_{X}^\top \overline{W} S_X\right)^{-1} (\lambda I+ L_K)\right\|_{op}^{1/2}\left(\lambda^{1/2} \left\|L_K-S_{X}^\top \overline{W} S_X\right\|_{op}+\lambda^{3/2}\right)\\
	&\le C_2 n^{-\frac{3/2}{3+\beta+\alpha(1-\beta)}}  \left(\log\frac{4}{\delta}\right)^3,
\end{align*}
where
\begin{align*}
	C_2&=2b\left(\sqrt{8\kappa^{2+\alpha}(C_0)^{1-\alpha}\sigma^2}+4(2C+\sigma)\kappa^2+1\right)^2\left(\sqrt{{8M^2\kappa^{2\alpha}(C_0)^{1-\alpha}\sigma^2}}+4M\kappa(3C+2\sigma)\right)\\
	&\qquad+C_r\left(\sqrt{8\kappa^{2+\alpha}(C_0)^{1-\alpha}\sigma^2}+4(2C+\sigma)\kappa^2+1\right)\left(8\kappa^2(C+\sigma)+1\right).
\end{align*}
Then the desired results holds by scaling $2\delta$ and $3\delta$ to $\delta$ respectively, taking $C'=\max\{C_1,C_2\}$  and $\log{\frac{12}{\delta}}>\log{\frac{8}{\delta}}>1.$
\subsection{Convergence analysis of spectral algorithms with clipped weights}
In this subsection, we will prove the convergence results for the spectral algorithm with clipped weights.
In the following, we will estimate $J_1, J_2, J_3, J_4$ and $J_5$ respectively. To this end, we need the following Bernstein inequality for vector-valued random variables, as presented  in \citet{Pinelis94}.
\begin{lemma}\label{lemma: bernstein inequality for bounded random variables}
	For a random variable  $\xi$ on $(Z; \rho)$ with values in a Hilbert space $(\H; \|\cdot\|)$
	satisfying $\|\xi\|\le \tilde{M}<\infty$ almost surely, and a random sample $\{z_i\}_{i=1}^s$  independent drawn
	according to $\rho$, there holds with confidence $1-\delta$,
	\begin{align*}
		\left\|\frac1s \sum_{i=1}^s\xi(z_i) -\mathbb{E}(\xi)\right\|\le \frac{2\tilde{M}\log\frac{2}{\delta}}{s}+\sqrt{\frac{2\mathbb{E}(\|\xi\|^2)\log\frac{2}{\delta}}{s}}.
	\end{align*}
\end{lemma}
The following proposition provides estimates for the operator norm of the operator
$$\left(\lambda I+\hat{L}_K\right)^{-1/2}\left(\hat{L}_K-S_X^\top \hat{W} S_X\right),$$  which is crucial to our proof.
\begin{proposition}\label{proposition: operator difference with effective dimension 2}
	For any $0<\delta<1,$	with confidence at least $1-\delta,$ there holds
	\begin{align*}
		\left\|	\left(\lambda I+\hat{L}_K\right)^{-1/2} \left(\hat{L}_K-S_X^\top \hat{W} S_X\right)\right\|_{op}\le  \frac{2\kappa^2 D_n\log\frac{2}{\delta}}{n\sqrt{\lambda}}+\sqrt{\frac{2\kappa^2 D_n\hat{\mathcal{N}}(\lambda)\log\frac{2}{\delta}}{n}},
	\end{align*}
	where $\hat{\mathcal{N}}(\lambda)={\rm Tr}(\hat{L}_K(\lambda I+\hat{L}_K)).$
\end{proposition}
\begin{proof}
	We consider the random variable $$\xi(x)=\left(\lambda I+\hat{L}_K\right)^{-1/2}\hat{w}(x)\langle \cdot, K_x\rangle K_x$$ which takes values in HS($\H_K$), the Hilbert space of Hilbert-Schmidt (HS) operators on $\H_K$ with inner product $\langle A,B \rangle_{HS}={\rm Tr}(B^\top A).$ The Hilbert Schmidt norm is given by $\|A\|_{HS}=\sum_{i} \|Ae_i\|_K^2$ where $\{e_i\}$ is an orthonormal basis of $\H_K,$ and we have the norm relations $\|A\|_{op}\le \|A\|_{HS}.$ Moreover,
	\begin{align*}
		\mathbb{E}_{x\sim\rho_\X^{tr}}\left[\xi(x)\right]=\mathbb{E}_{x\sim\rho_\X^{tr}}\left[\left(\lambda I+\hat{L}_K\right)^{-1/2} \hat{w}(x) \langle K_x,\cdot\rangle K_x\right]=\left(\lambda I+\hat{L}_K\right)^{-1/2} \hat{L}_K,
	\end{align*}
	then $$\left(\lambda I+ \hat{L}_K\right)^{-1/2}\left(S_{X}^\top \hat{W} S_X-\hat{L}_K\right)=\frac{1}{n}\sum_{i=1}^n \xi(x_i)-\mathbb{E}_{x\sim\rho_\X^{tr}}[\xi(x)].$$
	Let $\{(\lambda_i,
	\phi_i)\}_{i}$ be a set of normalized eigenpairs of $L_K$ on
	$\mathcal H_K$ with $\{\phi_i\}_{i=1}^\infty$ forming an
	orthonormal basis of $\mathcal H_K$, then by the Mercer Theorem, we have
	\begin{align*}
		K(x,x')=\sum_{i=1}^\infty \phi_i(x)\phi_i(x'), \quad \forall x,x'\in\X.
	\end{align*}
	Moreover, by the reproducing property, we have  $\langle K_x,\cdot\rangle K_x \phi_i=\phi_i(x)K_x$ and $K_x=\sum_{\ell=1}^\infty\langle K_x,\phi_\ell\rangle \phi_\ell =\sum_{\ell=1}^\infty\phi_\ell(x)\phi_\ell.$ Then the definition of the Hilbert-Schmidt norm implies that
	\begin{align*}
		&\|\xi(x)\|_{HS}^2=\left\|\left(\lambda I+\hat{L}_K\right)^{-1/2} \hat{w}(x) \langle K_x,\cdot\rangle K_x\right\|_{HS}^2=\sum_{i=1}^\infty\left\|\left(\lambda I+\hat{L}_K\right)^{-1/2} \hat{w}(x) \langle K_x,\cdot\rangle K_x \phi_i\right\|_K^2\\
		&=\left|\hat{w}(x) \right|^2\sum_{i=1}^\infty\left\|\left(\lambda I+\hat{L}_K\right)^{-1/2} \phi_i(x)\sum_{\ell=1}^\infty\phi_\ell(x)\phi_\ell\right\|_K^2\le D_n^2 \sum_{i=1}^\infty\left\|\left(\lambda I+\hat{L}_K\right)^{-1/2} \phi_i(x)\sum_{\ell=1}^\infty\phi_\ell(x)\phi_\ell\right\|_K^2\\&
		=D_n^2\sum_{i=1}^\infty (\phi_i(x))^2\left\| \sum_{\ell=1}^\infty\phi_\ell(x) \frac{1}{\sqrt{\lambda+\lambda_\ell}}\phi_\ell\right\|_K^2=D_n^2\sum_{i=1}^\infty (\phi_i(x))^2 \sum_{\ell=1}^\infty \frac{(\phi_\ell(x))^2}{\lambda+\lambda_\ell}\\
		&
		=D_n^2K(x,x) \sum_{\ell=1}^\infty \frac{(\phi_\ell(x))^2}{\lambda+\lambda_\ell}\le \kappa^2D_n^2\sum_{\ell=1}^\infty \frac{(\phi_\ell(x))^2}{\lambda+\lambda_\ell}.
	\end{align*}								
	Then we have $$\|\xi(x)\|_{HS}^2\le\kappa^2 D_n^2\sum_{\ell=1}^\infty \frac{(\phi_\ell(x))^2}{\lambda+\lambda_\ell}\le \kappa^2 D_n^2\lambda^{-1} \sum_{\ell=1}^\infty (\phi_\ell(x))^2=\kappa^2 D_n^2 \lambda^{-1} K(x,x)\le   \kappa^4\lambda^{-1}D_n^2.$$ and
	\begin{align*}
		&\mathbb{E}_{x\sim\rho_\X^{tr}}\left[\|\xi(x)\|_{HS}^2\right]=\mathbb{E}_{x\sim\rho_\X^{tr}}\left[\left\|\left(\lambda I+\hat{L}_K\right)^{-1/2}\hat{w}(x)\langle \cdot, K_x\rangle K_x\right\|_{HS}^2\right]\\
		&=\mathbb{E}_{x\sim\rho_\X^{tr}}\left[{\rm Tr}\left(\left(\left(\lambda I+\hat{L}_K\right)^{-1/2}\hat{w}(x)\langle \cdot, K_x\rangle K_x\right)^\top \left(\lambda I+\hat{L}_K\right)^{-1/2}\hat{w}(x)\langle \cdot, K_x\rangle K_x\right)\right]\\
		&=\mathbb{E}_{x\sim\rho_\X^{tr}}\left[{\rm Tr}\left(\left(\lambda I+\hat{L}_K\right)^{-1}\hat{w}^2(x)K(x,x)\langle \cdot, K_x\rangle K_x\right)\right]\\
		&\le \kappa^2 D_n\hat{\mathcal{N}}(\lambda).
	\end{align*}
	Then applying Lemma \ref{lemma: bernstein inequality for bounded random variables} to the random variable $\xi(x)$, with confidence at least $1-\delta$, we have
	\begin{align*}
		\left\|	\left(\lambda I+\hat{L}_K\right)^{-1/2} \left(\hat{L}_K-S_X^\top \hat{W} S_X\right)\right\|_{op}&\le \left\|	\left(\lambda I+\hat{L}_K\right)^{-1/2}\left(\hat{L}_K-S_X^\top \hat{W} S_X\right)\right\|_{HS}\\
		&\le \frac{2\kappa^2 D_n \log\frac{2}{\delta}}{n\sqrt{\lambda}}+\sqrt{\frac{2\kappa^2 D_n \hat{\mathcal{N}}(\lambda)\log\frac{2}{\delta}}{n}}.
	\end{align*}
	This completes the proof.
\end{proof}
Now we are prepared to estimate $J_1,$ $J_2$, $J_3$, $J_4$ and $J_5$ in the following propositions.
\begin{proposition}\label{proposition: J2}
	For any $0<\delta<1,$	with confidence at least $1-\delta,$ there holds
	\begin{align*}
		J_2=\left\|	\left(\lambda I+\hat{L}_K\right)^{-1}\left(\lambda I+S_X^\top \hat{W}S_X\right)\right\|_{op}\le \frac{2\kappa^2 D_n\log\frac{2}{\delta}}{n\lambda}+\sqrt{\frac{2\kappa^2 D_n\hat{\mathcal{N}}(\lambda)\log\frac{2}{\delta}}{n{\lambda}}}+1.
	\end{align*}
\end{proposition}
\begin{proof}

	Let $A$ and $B$ be invertible operators on a Banach space, we have the following identity $A^{-1}B=(A^{-1}-B^{-1}+B^{-1})B=(A^{-1}-B^{-1})B+I=A^{-1}(B-A)B^{-1}B+I=A^{-1}(B-A)+I.$ Applying the above identity above to $A=\lambda I+\hat{L}_K$ and  $B=\lambda I+ S_{X}^\top \hat{W} S_X$, we obtain
	%			  \begin{align*}
		%			  		\left(\lambda I+\hat{L}_K\right)^{-1}\left(\lambda I+S_X^\top \hat{W}S_X\right)=	\left(\lambda I+\hat{L}_K\right)^{-1} \left(S_X^\top \hat{W}S_X-\hat{L}_K  \right) +I,
		%			  		\end{align*}
	%			  		 we have
	\begin{align*}
		J_2&=\left\|	\left(\lambda I+\hat{L}_K\right)^{-1}\left(\lambda I+S_X^\top \hat{W}S_X\right)\right\|_{op}\\
		&=\left\|	\left(\lambda I+\hat{L}_K\right)^{-1} \left(S_X^\top \hat{W}S_X-\hat{L}_K  \right) +I\right\|_{op}\\
		&\le\lambda^{-1/2}\cdot \left\| \left(\lambda I+ \hat{L}_K\right)^{-1/2}\left(S_{X}^\top \hat{W} S_X-\hat{L}_K\right)\right\|_{op}+1,
	\end{align*}
	where the inequality holds due to the fact $\left\|\left(\lambda I+ \hat{L}_K\right)^{-1/2}\right\|_{op}\le \lambda^{-\frac12}$.
	By Proposition \ref{proposition: operator difference with effective dimension 2}, with confidence at least $1-\delta,$ the following bound holds
	\begin{align*}
		J_2=\left\|	\left(\lambda I+\hat{L}_K\right)^{-1}\left(\lambda I+S_X^\top \hat{W}S_X\right)\right\|_{op}
		\le\frac{2\kappa^2 D_n\log\frac{2}{\delta}}{n\lambda}+\sqrt{\frac{2\kappa^2 D_n\hat{\mathcal{N}}(\lambda)\log\frac{2}{\delta}}{n{\lambda}}}+1.
	\end{align*}
	This completes the proof.
\end{proof}

\begin{proposition}\label{proposition: J3}
	With confidence at least $1-\delta$, there holds
	\begin{align*}
		J_3=\left\|	\left(\lambda I+\hat{L}_K\right)^{-1/2}\left(S_X^\top \hat{W} \bar{y}-S_X^\top \hat{W} f_\rho\right)\right\|_K\le \frac{4M\kappa D_n\log\frac{2}{\delta}}{n\sqrt{\lambda}}+\sqrt{\frac{8M^2 D_n \hat{\mathcal{N}}(\lambda)\log\frac{2}{\delta}}{n}}.
	\end{align*}
\end{proposition}
\begin{proof}
	We consider the random variable $\xi_3(x,y)=\left(\lambda I+\hat{L}_K\right)^{-1/2}\hat{w}(x)\left(y-f_\rho(x)\right)K(\cdot,x)$, then $\left\|\xi_3(x,y)\right\|_{K}\le 2M\kappa\lambda^{-1/2}D_n$ and
	\begin{align*}
		&\mathbb{E}_{(x,y)\sim\rho^{tr}}\left[\|\xi_3(x,y)\|_K^2\right]=\mathbb{E}_{(x,y)\sim\rho^{tr}}\left[\left\|\left(\lambda I+\hat{L}_K\right)^{-1/2}\hat{w}(x)\left(y-f_\rho(x)\right)K(\cdot,x)\right\|_K^2\right]\\
		&\le 4M^2\mathbb{E}_{(x,y)\sim\rho^{tr}}\left[\left\|\left(\lambda I+\hat{L}_K\right)^{-1/2}\hat{w}(x)K(\cdot,x)\right\|_K^2\right]\\
		&=4M^2\mathbb{E}_{(x,y)\sim\rho^{tr}}{\rm Tr}\left[\left(\left(\lambda I+\hat{L}_K\right)^{-1/2}\hat{w}(x)K(\cdot,x)\otimes \left(\lambda I+\hat{L}_K\right)^{-1/2}\hat{w}(x)K(\cdot,x)\right)\right]\\
		&=4M^2 {\rm Tr}\left(\mathbb{E}_{(x,y)\sim\rho^{tr}}\left[\left(\lambda I+\hat{L}_K\right)^{-1/2}\hat{w}(x)K(\cdot,x)\otimes \left(\lambda I+\hat{L}_K\right)^{-1/2}\hat{w}(x)K(\cdot,x)\right]\right)\\
		&\le 4M^2 D_n {\rm Tr}\left(\left(\lambda I+\hat{L}_K\right)^{-1}\hat{L}_K\right)\\
		&=4M^2 D_n \hat{\mathcal{N}}(\lambda).
	\end{align*}
	Then by Lemma \ref{lemma: bernstein inequality for bounded random variables}, with confidence at least $1-\delta$, there holds
	\begin{align*}
		J_3=\left\|	\left(\lambda I+\hat{L}_K\right)^{-1/2} \left(S_X^\top \hat{W} \bar{y}-S_X^\top \hat{W} f_\rho\right)\right\|_K\le \frac{4M\kappa D_n\log\frac{2}{\delta}}{n\sqrt{\lambda}}+\sqrt{\frac{8M^2 D_n \hat{\mathcal{N}}(\lambda)\log\frac{2}{\delta}}{n}}.
	\end{align*}
	This completes the proof.
\end{proof}
\begin{proposition}\label{proposition: J5}
	With confidence at least $1-\delta,$ there holds
	\begin{align*}
		J_5=\left\|	\hat{L}_K-S_X^\top \hat{W} S_X \right\|_{op}\le \frac{2\kappa^2 D_n \log\frac{2}{\delta}}{n}+\sqrt{\frac{2\kappa^4 D_n^2 \log\frac{2}{\delta}}{n}}.
	\end{align*}
\end{proposition}
\begin{proof}
	We consider the random variable $\xi_5(x)=\hat{w}(x)\langle\cdot,K_x \rangle K_x$ which takes values in HS($\H_K$),
	then  $\|\xi_5(x)\|_{HS}\le \kappa^2 D_n$ and
	\begin{align*}
		\mathbb{E}_{x\sim\rho_\X^{tr}}\left[\left\|\xi_5(x)\right\|_{HS}^2\right]&=\mathbb{E}_{x\sim\rho_\X^{tr}}\left[\left\|\hat{w}(x)\langle \cdot, K_x\rangle K_x\right\|_{HS}^2\right]\\
		&=\mathbb{E}_{x\sim\rho_\X^{tr}}\left[{\rm Tr}((\hat{w}(x)\langle \cdot, K_x\rangle K_x)^\top\hat{w}(x)\langle \cdot, K_x\rangle K_x)\right]\\
		&=\mathbb{E}_{x\sim\rho_\X^{tr}}\left[{\rm Tr}\left(\hat{w}^2(x)K(x,x)\langle \cdot, K_x\rangle K_x\right)\right]\\
		&\le \kappa^2 D_n {\rm Tr}(\hat{L}_K)\\
		&\le \kappa^4 D_n^2.
	\end{align*}
	By applying  Lemma \ref{lemma: bernstein inequality for bounded random variables} to the random variable $\xi_5(x)$, with confidence at least $1-\delta$, we can conclude that
	\begin{align*}
		J_5=\left\|\hat{L}_K-S_X^\top \hat{W} S_X\right\|_{op}\le \left\|	\hat{L}_K-S_X^\top \hat{W} S_X\right\|_{HS}
		\le \frac{2\kappa^2 D_n \log\frac{2}{\delta}}{n}+\sqrt{\frac{2\kappa^4 D_n^2 \log\frac{2}{\delta}}{n}}.
	\end{align*}
	The proof is now complete.
\end{proof}
\begin{proposition}\label{proposition: J4}
	Under Assumption \ref{assumption: weight assumption} with $0<\alpha\le 1$, for any $k\in\mathbb{N}$ and $k\ge 2,$ we have
	\begin{align*}
		J_4=\left\|	L_K- \hat{L}_K\right\|_{op}\le \kappa^2 D_n^{-\frac{k-1}{\alpha}} \left( \frac12 k!C^{k-2}\sigma^2\right)^\frac{1}{\alpha}.
	\end{align*}
\end{proposition}
\begin{proof}
	First, 	for any $k\in\mathbb{N}$, $k\ge 2$ and $0<\alpha\le 1$, we have
	\begin{align*}
		\mathbb{I}_{\{w(x)\ge D_n\}}\le \left(\frac{w(x)}{D_n}\right)^{\frac{k-1}{\alpha}}.
	\end{align*}
	Then, for any $f\in\H_K$,
	\begin{align*}
		J_4&=\left\|L_K- \hat{L}_K\right\|_{op}=\sup_{\|f\|_K=1}\left\|\left(L_K- \hat{L}_K\right)f\right\|_K\\
		&=\sup_{\|f\|_K=1}\left\| \int_{\X} f(x)K_x d\rho_\X^{te}-\int_\X \hat{w}(x) f(x) K_x d\rho_\X^{tr} \right\|_K\\
		&=\sup_{\|f\|_K=1}\left\| \int_{\X}w(x) f(x)K_x d\rho_\X^{tr}-\int_\X \hat{w}(x) f(x)K_xd\rho_\X^{tr} \right\|_K\\
		&=\sup_{\|f\|_K=1}\left\| \int_{\X}\left(w(x)-\hat{w}(x)\right) f(x)K_x d\rho_\X^{tr} \right\|_K\\
		&\le \sup_{\|f\|_K=1}\|f\|_K^2 \kappa^2 \int_{\X}|w(x)-\hat{w}(x)| d\rho_\X^{tr} \\
		&\le \kappa^2  \int_{\X} w(x) \mathbb{I}_{\{w(x)\ge D_n\}} d\rho_\X^{tr} \\
		&\le \kappa^2\int_{\X} w(x) (w(x) )^{\frac{k-1}{\alpha}} D_n^{-\frac{k-1}{\alpha}} d\rho_\X^{tr}= \kappa^2 D_n^{-\frac{k-1}{\alpha}}\int_{\X}(w(x) )^{\frac{k-1}{\alpha}}  d\rho_\X^{te}\le \kappa^2 D_n^{-\frac{k-1}{\alpha}} \left( \frac12 k!C^{k-2}\sigma^2\right)^\frac{1}{\alpha},
	\end{align*}
	where the last inequality  follows from Assumption \ref{assumption: weight assumption} with $0<\alpha\le 1.$ Then the proof is now finished.
\end{proof}
By Proposition \ref{proposition: J4}, we can estimate $J_1$ as follows.
\begin{proposition}\label{proposition: J1}
	Under Assumption \ref{assumption: weight assumption} with $0<\alpha\leq 1$, for any $k\in\mathbb{N}$ and $k\ge 2,$ we have
	\begin{align*}
		J_1=\left\|L_K \left(\lambda I+ \hat{L}_K\right)^{-1}\right\|_{op}\le \kappa^2 \lambda^{-1} D_n^{-\frac{k-1}{\alpha}} \left( \frac12 k!C^{k-2}\sigma^2\right)^\frac{1}{\alpha}+1.
	\end{align*}
\end{proposition}
\begin{proof} Initially, we observe that
	\begin{align*}
		J_1&=\left\|L_K \left(\lambda I+ \hat{L}_K\right)^{-1}\right\|_{op}=\left\|L_K\left[\left(\lambda I+\hat{L}_K\right)^{-1}-\left(\lambda I+{L}_K\right)^{-1}\right]+L_K\left(\lambda I+{L}_K\right)^{-1}\right\|_{op}\\
		&=\left\|L_K\left(\lambda I+{L}_K\right)^{-1}\left[\left(\lambda I+{L}_K\right)-\left(\lambda I+\hat{L}_K\right)\right] \left(\lambda I+\hat{L}_K\right)^{-1}+L_K\left(\lambda I+{L}_K\right)^{-1}\right\|_{op}\\
		&\le \left\|L_K\left(\lambda I+{L}_K\right)^{-1}\right\|_{op}\cdot \left\|{L}_K-\hat{L}_K\right\|_{op} \lambda^{-1}+\left\|L_K\left(\lambda I+{L}_K\right)^{-1}\right\|_{op}\\
		&\le \left\|{L}_K-\hat{L}_K\right\|_{op} \lambda^{-1}+1.
	\end{align*}
	Then the desired result holds due to Proposition \ref{proposition: J4}.
\end{proof}
The following proposition describes the relationship between $\hat{\mathcal{N}}(\lambda)$ and $\mathcal{N}(\lambda)$. Let $A$ and $B$ be self-adjoint operators on a Hilbert space $\mathcal{H}$. The notation $A\succeq B$ indicates that $A-B\succeq 0$, where $A-B$ is a positive semidefinite operator.
\begin{proposition}\label{proposition: relatioship between effecitve dimension}
	For any $\lambda>0,$ we have
	\begin{align}\label{eq: relatioship between effecitve dimension}
		\hat{\mathcal{N}}(\lambda)\le \mathcal{N}(\lambda).
	\end{align}
\end{proposition}
\begin{proof}
	On one hand, for any $f\in\H_K,$ we have
	\begin{align*}
		\left\langle(L_K- \hat{L}_K)f,f\right\rangle_K&=\left\langle  \int_{\X}w(x) f(x)K_x d\rho_\X^{tr}-\int_\X \hat{w}(x) f(x)K_xd\rho_\X^{tr},f\right\rangle_K\\
		&= \int_{\X} f^2(x)\left(w(x)-\hat{w}(x)\right) d\rho_\X^{tr}\ge0,
	\end{align*}
	which implies that $L_K\succeq \hat{L}_K$, it follows that $\left(\lambda I+\hat{L}_K\right)^{-1}\succeq \left(\lambda I+L_K\right)^{-1}.$
	
	On the other hand, since
	${L}_K(\lambda I+{L}_K)^{-1}=I- \lambda(\lambda I+{L}_K)^{-1}$, then
	\begin{align*}
		{L}_K(\lambda I+{L}_K)^{-1} - \hat{L}_K\left(\lambda I+\hat{L}_K\right)^{-1}=\lambda \left(\left(\lambda I+\hat{L}_K\right)^{-1}-\left(\lambda I+{L}_K\right)^{-1}\right)\succeq 0.
	\end{align*}
	This completes the proof.
	%recall $\hat{\mathcal{N}}(\lambda)={\rm Tr}(\hat{L}_K(\lambda I+\hat{L}_K)^{-1}),$ and ${\mathcal{N}}(\lambda)={\rm Tr}({L}_K(\lambda I+{L}_K)^{-1}),$ then
\end{proof}
Now we are in a position to prove our first main result.

{\noindent \bf Proof of Theorem \ref{theorem: clipped weight spectral algorithm}.} To prove the theorem, we derive upper bounds for each of the five terms $J_1, J_2, J_3, J_4,$ and  $J_5$ appearing in  Proposition \ref{proposition: error decompsotion for clipped weight} through an appropriate choice of parameters.

First, when $1/2\le r\le 3/2,$	we take $\lambda=n^{-\frac{1}{2r+\beta}+\frac{\epsilon}{r}}$ with $0< \epsilon<\frac{r}{2r+\beta}.$ We choose $D_n=n^{\alpha\epsilon}$ and set $k-1$ to be the integer part of $\frac{1}{\epsilon},$ i.e., $k-1=\left\lceil\frac{1}{\epsilon}\right\rceil$,this ensures $1-\epsilon< (k-1)\epsilon\le 1$, Consequently,   $D_n^{-\frac{k-1}{\alpha}}=n^{-(k-1)\epsilon}\le n^{-1+\epsilon}$. Then, by Proposition \ref{proposition: J1}, we have
\begin{align*}
	J_1&=\left\|L_K \left(\lambda I+ \hat{L}_K\right)^{-1}\right\|_{op}\le \kappa^2 \lambda^{-1} D_n^{-\frac{k-1}{\alpha}} \left( \frac12 k!C^{k-2}\sigma^2\right)^\frac{1}{\alpha}+1\\
	&\le \kappa^2 n^{-1+\epsilon}n^{\frac{1}{2r+\beta}-\frac{\epsilon}{r}}  \left(\frac12\left(\left\lceil\frac{1}{\epsilon}\right\rceil+1\right)!C^{\left\lceil\frac{1}{\epsilon}\right\rceil-1}\sigma^2\right)^\frac{1}{\alpha}+1\\
	&\le \kappa^2   \left(\frac12\left(\left\lceil\frac{1}{\epsilon}\right\rceil+1\right)!C^{\left\lceil\frac{1}{\epsilon}\right\rceil-1}\sigma^2\right)^\frac{1}{\alpha}+1,
\end{align*}
where the last inequality follows from $ n^{-1+\epsilon+\frac{1}{2r+\beta}-\frac{\epsilon}{r}}\le 1$ for $1/2\le r\le 3/2.$
By Proposition \ref{proposition: J2}, with confidence at least $1-\delta/2,$ we have
\begin{align*}
	J_2&=\left\|\left(\lambda I+\hat{L}_K\right)^{-1}\left(\lambda I+S_X^\top \hat{W}S_X\right)\right\|_K\\
	& \le \frac{2\kappa^2 D_n\log\frac{4}{\delta}}{n\lambda}+\sqrt{\frac{2\kappa^2 D_n\hat{\mathcal{N}}(\lambda)\log\frac{4}{\delta}}{n{\lambda}}}+1\\
	&\le \left(\frac{2\kappa^2 n^{\alpha\epsilon}}{ n^{1-\frac{1}{2r+\beta}+\frac{\epsilon}{r}}}+\sqrt{\frac{2\kappa^2 n^{\alpha\epsilon} C_0 n^{\frac{\beta}{2r+\beta}-\frac{\beta\epsilon}{r}}}{n^{1-\frac{1}{2r+\beta}+\frac{\epsilon}{r}}}}+1\right) \log\frac{4}{\delta}\\
	&\le \left(2\kappa^2 +\sqrt{2\kappa^2  C_0 }+1\right)\log\frac{4}{\delta},
\end{align*}
where the last inequality holds because $n^{\alpha\epsilon-1+\frac{1}{2r+\beta}-\frac{\epsilon}{r}}\le 1$, $n^{\alpha\epsilon+\frac{\beta}{2r+\beta}-\frac{\beta\epsilon}{r}-1+\frac{1}{2r+\beta}-\frac{\epsilon}{r}}\le 1,$ $\log\frac{4}{\delta}>1$ for $\delta\in(0,1).$ Furthermore, according to Proposition \ref{proposition: J3}, with confidence at least $1-\delta/2,$ there holds
\begin{align*}
	J_3&=\left\|\left(\lambda I+\hat{L}_K\right)^{-1/2}\left(S_X^\top \hat{W} \bar{y}-S_X^\top \hat{W} f_\rho\right)\right\|_K\\
	&\le \frac{4M\kappa D_n\log\frac{4}{\delta}}{n\sqrt{\lambda}}+\sqrt{\frac{8M^2 D_n \hat{\mathcal{N}}(\lambda)\log\frac{4}{\delta}}{n}}\\
	&\le \frac{4M\kappa n^{\alpha\epsilon}\log\frac{4}{\delta}}{n^{1-\frac{1}{2(2r+\beta)}+\frac{\epsilon}{2r}}}+\sqrt{\frac{8M^2n^{\alpha\epsilon} C_0 n^{\frac{\beta}{2r+\beta}-\frac{\beta\epsilon}{r}}\log\frac{4}{\delta}}{n}}\\
	&\le \left(4M\kappa +\sqrt{8M^2 C_0 }\right)n^{-\frac{r}{2r+\beta}+\epsilon}\log\frac{4}{\delta}.
\end{align*}
Putting the above estimates back into Proposition \ref{proposition: error decompsotion for clipped weight}, for $1/2\le r\le 3/2,$ with confidence at least $1-2\delta$, we have
\begin{equation*}
	\begin{aligned}
		\|f_{\bz,\lambda}^{\hat{\mathbf{w}}}-f_\rho\|_{\rho_\X^{te}}&\le   2bJ_1^{1/2} J_2 J_3+2^r(b+1+\gamma_r)\left\| u_\rho\right\|_{\rho_\X^{te}}\lambda^r J_1^{r} J_2^{r}\\
		&\le 2b\left(\kappa^2   \left(\frac12\left(\left\lceil\frac{1}{\epsilon}\right\rceil+1\right)!C^{\left\lceil\frac{1}{\epsilon}\right\rceil-1}\sigma^2\right)^\frac{1}{\alpha}+1\right)^{1/2} \cdot   \left(2\kappa^2 +\sqrt{2\kappa^2  C_0 }+1\right)\\
		&\cdot  \left(4M\kappa +\sqrt{8M^2 C_0 }\right)n^{-\frac{r}{2r+\beta}+\epsilon}\log^2\frac{4}{\delta}\\&+2^r(b+1+\gamma_r)\left\| u_\rho\right\|_{\rho_\X^{te}}n^{-\frac{r}{2r+\beta}+\epsilon}  \left(\kappa^2   \left(\frac12\left(\left\lceil\frac{1}{\epsilon}\right\rceil+1\right)!C^{\left\lceil\frac{1}{\epsilon}\right\rceil-1}\sigma^2\right)^\frac{1}{\alpha}+1\right)^{r}\\ &\cdot\left(2\kappa^2 +\sqrt{2\kappa^2  C_0 }+1\right)^{2r}\log^{r}\frac{4}{\delta}\\
		&\le C_3 n^{-\frac{r}{2r+\beta}+\epsilon}\log^{2}\frac{4}{\delta}.
	\end{aligned}
\end{equation*}
where
\begin{align*}
	C_3&= 2b\left(\kappa^2   \left(\frac12\left(\left\lceil\frac{1}{\epsilon}\right\rceil+1\right)!C^{\left\lceil\frac{1}{\epsilon}\right\rceil-1}\sigma^2\right)^\frac{1}{\alpha}+1\right)^{1/2} \cdot   \left(2\kappa^2 +\sqrt{2\kappa^2  C_0 }+1\right)  \left(4M\kappa +\sqrt{8M^2 C_0 }\right)\\
	&+2^r(b+1+\gamma_r)\left\| u_\rho\right\|_{\rho_\X^{te}}  \left(\kappa^2   \left(\frac12\left(\left\lceil\frac{1}{\epsilon}\right\rceil+1\right)!C^{\left\lceil\frac{1}{\epsilon}\right\rceil-1}\sigma^2\right)^\frac{1}{\alpha}+1\right)^{r} \cdot\left(2\kappa^2 +\sqrt{2\kappa^2  C_0 }+1\right)^{2r}.
\end{align*}

We now consider the case when $r> 3/2,$ in this case, we select $\lambda=n^{-\frac{1}{2r+\beta}+\frac{\epsilon}{r}}$ with $0< \epsilon<\frac{r}{2r+\beta}.$ let's set $D_n=n^{\frac{\alpha\epsilon}{r-1/2}}$ and take $k-1$ to be the integer part of $\frac{r-1/2}{\epsilon}$, this yields  $D_n^{-\frac{k-1}{\alpha}}=n^{-\frac{(k-1)\epsilon}{r-1/2}}\le n^{-1+\frac{\epsilon}{r-1/2}}$. With these choice, we may apply Proposition \ref{proposition: J1} to obtain the following result for $J_1$,
\begin{align*}
	J_1&=\left\|L_K \left(\lambda I+ \hat{L}_K\right)^{-1}\right\|_{op}\\
	&\le \kappa^2 \lambda^{-1} D_n^{-\frac{k-1}{\alpha}} \left( \frac12 k!C^{k-2}\sigma^2\right)^\frac{1}{\alpha}+1\\
	&\le \kappa^2 n^{-1+\frac{\epsilon}{r-1/2}}n^{\frac{1}{2r+\beta}-\frac{\epsilon}{r}}  \left(\frac12\left(\left\lceil\frac{r-1/2}{\epsilon}\right\rceil+1\right)!C^{\lceil\frac{r-1/2}{\epsilon}\rceil-1}\sigma^2\right)^\frac{1}{\alpha}+1\\
	&\le \kappa^2  \left(\frac12\left(\left\lceil\frac{r-1/2}{\epsilon}\right\rceil+1\right)!C^{\lceil\frac{r-1/2}{\epsilon}\rceil-1}\sigma^2\right)^\frac{1}{\alpha}+1,
\end{align*}
where we can verify that $n^{-1+\frac{\epsilon}{r-1/2}}n^{\frac{1}{2r+\beta}-\frac{\epsilon}{r}}\le 1.$

By Proposition \ref{proposition: J2}, with confidence at least $1-\delta/3,$ we have
\begin{align*}
	J_2&=\left\|\left(\lambda I+\hat{L}_K\right)^{-1}\left(\lambda I+S_X^\top \hat{W}S_X\right)\right\|_K\\
	&\le \frac{2\kappa^2 D_n\log\frac{6}{\delta}}{n\lambda}+\sqrt{\frac{2\kappa^2 D_n\hat{\mathcal{N}}(\lambda)\log\frac{6}{\delta}}{n{\lambda}}}+1\\
	&\le \frac{2\kappa^2 n^{\frac{\alpha\epsilon}{r-1/2}}\log\frac{6}{\delta}}{ n^{1-\frac{1}{2r+\beta}+\frac{\epsilon}{r}}}+\sqrt{\frac{2\kappa^2 n^{\frac{\alpha\epsilon}{r-1/2}} C_0 n^{\frac{\beta}{2r+\beta}-\frac{\beta\epsilon}{r}}\log\frac{6}{\delta}}{n^{1-\frac{1}{2r+\beta}+\frac{\epsilon}{r}}}}+1\\
	&\le \left(2\kappa^2 +\sqrt{2\kappa^2  C_0 }+1\right) \log\frac{6}{\delta},
\end{align*}
where we can verify that $n^{-1+\frac{\alpha\epsilon}{r-1/2}}n^{\frac{1}{2r+\beta}-\frac{\epsilon}{r}}\le 1$ and $n^{\frac{\alpha\epsilon}{r-1/2}}n^{1-\frac{1}{2r+\beta}+\frac{\epsilon}{r}}n^{-1+\frac{1}{2r+\beta}-\frac{\epsilon}{r}}\le 1.$

According to Proposition \ref{proposition: J3}, with confidence at least $1-\delta/3,$ there holds
\begin{align*}
	J_3&=\left\|\left(\lambda I+\hat{L}_K\right)^{-1/2}\left(S_X^\top \hat{W} \bar{y}-S_X^\top \hat{W} f_\rho\right)\right\|_K\\
	&\le \frac{4M\kappa D_n\log\frac{6}{\delta}}{n\sqrt{\lambda}}+\sqrt{\frac{8M^2 D_n \hat{\mathcal{N}}(\lambda)\log\frac{6}{\delta}}{n}}\\
	&\le \frac{4M\kappa n^{\frac{\alpha\epsilon}{r-1/2}}\log\frac{6}{\delta}}{n^{1-\frac{1}{2(2r+\beta)}+\frac{\epsilon}{2r}}}+\sqrt{\frac{8M^2 n^{\frac{\alpha\epsilon}{r-1/2}} C_0 n^{\frac{\beta}{2r+\beta}-\frac{\beta\epsilon}{r}}\log\frac{6}{\delta}}{n}}\\
	&\le \left(4M\kappa +\sqrt{8M^2 C_0 }\right)n^{-\frac{r}{2r+\beta}+\epsilon}\log\frac{6}{\delta},
\end{align*}
where we can easily verify that $n^{\frac{\alpha\epsilon}{r-1/2}}n^{-1+\frac{1}{2(2r+\beta)}-\frac{\epsilon}{2r}}\le n^{-
	\frac{r}{2r+\beta}+\epsilon}$, and $n^{\frac{\alpha\epsilon}{r-1/2}}n^{\frac{\beta}{2r+\beta}-\frac{\beta\epsilon}{r}} n^{-1}\le  n^{-
	\frac{r}{2r+\beta}+\epsilon}.$

By Proposition \ref{proposition: J4}, we can obtain
\begin{align*}
	J_4&=\left\|	L_K- \hat{L}_K\right\|_{op}\le  \kappa^2 D_n^{-\frac{k-1}{\alpha}} \left( \frac12 k!C^{k-2}\sigma^2\right)^\frac{1}{\alpha}\\
	&\le \kappa^2  \left(\frac12\left(\left\lceil\frac{r-1/2}{\epsilon}\right\rceil+1\right)!C^{\left\lceil\frac{r-1/2}{\epsilon}\right\rceil-1}\sigma^2\right)^\frac{1}{\alpha} n^{-1+\frac{\epsilon}{r-1/2}}.
\end{align*}
It follows that
\begin{align*}
	D_n^{r-3/2}\lambda^{1/2}J_4&=	n^{\frac{(r-3/2)\alpha\epsilon}{r-1/2}}n^{-\frac{1}{2(2r+\beta)}+\frac{\epsilon}{2r}}  \left\|	L_K- \hat{L}_K\right\|_{op}\\
	&\le n^{\frac{(r-3/2)\alpha\epsilon}{r-1/2}}n^{-\frac{1}{2(2r+\beta)}+\frac{\epsilon}{2r}} \kappa^2  \left(\frac12\left(\left\lceil\frac{r-1/2}{\epsilon}\right\rceil+1\right)!C^{\left\lceil\frac{r-1/2}{\epsilon}\right\rceil-1}\sigma^2\right)^\frac{1}{\alpha} n^{-1+\frac{\epsilon}{r-1/2}}\\
	&\le \kappa^2  \left(\frac12\left(\left\lceil\frac{r-1/2}{\epsilon}\right\rceil+1\right)!C^{\left\lceil\frac{r-1/2}{\epsilon}\right\rceil-1}\sigma^2\right)^\frac{1}{\alpha} n^{-1+\epsilon-\frac{1}{2(2r+\beta)}+\frac{\epsilon}{2r}} \\
	&\le\kappa^2  \left(\frac12\left(\left\lceil\frac{r-1/2}{\epsilon}\right\rceil+1\right)!C^{\left\lceil\frac{r-1/2}{\epsilon}\right\rceil-1}\sigma^2\right)^\frac{1}{\alpha} n^{-1+\epsilon},
\end{align*}
where the last inequality holds due to the fact that $-\frac{1}{2(2r+\beta)}+\frac{\epsilon}{2r}<0$ since $\epsilon<\frac{r}{2r+\beta}$.
And by Proposition \ref{proposition: J5}, with confidence at least $1-\delta/3$, we can conclude that
\begin{align*}
	J_5=\left\|\hat{L}_K-S_X^\top \hat{W} S_X\right\|_{op}
	\le \frac{2\kappa^2 D_n\log\frac{2}{\delta}}{n}+\sqrt{\frac{2\kappa^4 D_n^2 \log\frac{2}{\delta}}{n}}\le {4\kappa^2 }n^{-\frac12 +\frac{\alpha\epsilon}{r-1/2}}\log\frac{6}{\delta}.
\end{align*}
Together with the choice of $D_n$ and $\lambda$, with confidence $1-\delta/3$, we have
\begin{align*}
	D_n^{r-3/2}\lambda^{1/2} J_5&= n^{\frac{(r-3/2)\alpha\epsilon}{r-1/2}}n^{-\frac{1}{2(2r+\beta)}+\frac{\epsilon}{2r}}\left\|\hat{L}_K-S_X^\top \hat{W} S_X\right\|_{op}\\
	&\le {4\kappa^2 } n^{\frac{(r-3/2)\alpha\epsilon}{r-1/2}}n^{-\frac{1}{2(2r+\beta)}+\frac{\epsilon}{2r}} n^{-\frac12 +\frac{\alpha\epsilon}{r-1/2}}\log\frac{6}{\delta}\\
	&={4\kappa^2 } n^{-\frac{1}{2(2r+\beta)}+\frac{\epsilon}{2r}-\frac12 +\alpha\epsilon}\log\frac{6}{\delta}\\
	&\le {4\kappa^2 } n^{-\frac12 +\alpha\epsilon}\log\frac{6}{\delta},
\end{align*}
where the last inequality holds due to the fact that $-\frac{1}{2(2r+\beta)}+\frac{\epsilon}{2r}<0$ since $\epsilon<\frac{r}{2r+\beta}$.		
Therefore, for $r\ge 3/2,$ substituting the above-mentioned estimates into Proposition \ref{proposition: error decompsotion for clipped weight}, with confidence at least $1-\delta,$ we have
\begin{equation*}
	\begin{aligned}
		&\|f_{\bz,\lambda}^{\hat{\mathbf{w}}}-f_\rho\|_{\rho_\X^{te}}\le J_1^{1/2} J_2^{1/2}\Big(2b J_2^{1/2} J_3+ \sqrt{2}(b+1+\gamma_{1/2}) (r-1/2)\kappa^{2r-3}\left\| u_\rho\right\|_{\rho_\X^{te}}\lambda^{1/2}D_n^{2r-3}\left(J_4+J_5\right)
		\\&~~+2^r(b+1+\gamma_r)\left\| u_\rho\right\|_{\rho_\X^{te}}\lambda^r\Big)\\
		&\le \Big( \kappa^2  \left(\frac12\left(\left\lceil\frac{r-1/2}{\epsilon}\right\rceil+1\right)!C^{\left\lceil\frac{r-1/2}{\epsilon}\right\rceil-1}\sigma^2\right)^\frac{1}{\alpha}+1\Big)^{1/2} \left(2\kappa^2 +\sqrt{2\kappa^2  C_0 }+1\right)^{1/2} \log^{1/2}\frac{6}{\delta}\\
		&\cdot \Bigg(2b \left(2\kappa^2 +\sqrt{2\kappa^2  C_0 }+1\right)^{1/2}  \left(4M\kappa +\sqrt{8M^2 C_0 }\right)n^{-\frac{r}{2r+\beta}+\epsilon}\log^{3/2}\frac{6}{\delta}\\
		&+ \sqrt{2}(b+1+\gamma_{1/2}) (r-1/2)\kappa^{2r-3}\left\| u_\rho\right\|_{\rho_\X^{te}}\Big(\kappa^2  \left(\frac12\left(\left\lceil\frac{r-1/2}{\epsilon}\right\rceil+1\right)!C^{\left\lceil\frac{r-1/2}{\epsilon}\right\rceil-1}\sigma^2\right)^\frac{1}{\alpha} \\
		&\cdot n^{-1+\epsilon}+ {4\kappa^2 } n^{-\frac12 +\alpha\epsilon}\log\frac{6}{\delta}\Big)+2^r(b+1+\gamma_r)\left\| u_\rho\right\|_{\rho_\X^{te}}n^{-\frac{r}{2r+\beta}+\epsilon}\Bigg)\\
		&\le C_4 n^{-\frac{r}{2r+\beta}+\epsilon}\log^2\frac{6}{\delta},
	\end{aligned}
\end{equation*}
where
\begin{align*}
	C_4&=\Big( \kappa^2  \left(\frac12\left(\left\lceil\frac{r-1/2}{\epsilon}\right\rceil+1\right)!C^{\lceil\frac{r-1/2}{\epsilon}\rceil-1}\sigma^2\right)^\frac{1}{\alpha}+1\Big)^{1/2} \left(2\kappa^2 +\sqrt{2\kappa^2  C_0 }+1\right)^{1/2} \\
	&\cdot \Bigg(2b \left(2\kappa^2 +\sqrt{2\kappa^2  C_0 }+1\right)^{1/2}  \left(4M\kappa +\sqrt{8M^2 C_0 }\right)+ \sqrt{2}(b+1+\gamma_{1/2}) (r-1/2)\kappa^{2r-3}\left\| u_\rho\right\|_{\rho_\X^{te}}\\
	&\cdot\Big(\kappa^2  \left(\frac12\left(\left\lceil\frac{r-1/2}{\epsilon}\right\rceil+1\right)!C^{\left\lceil\frac{r-1/2}{\epsilon}\right\rceil-1}\sigma^2\right)^\frac{1}{\alpha}+ {4\kappa^2 }\Big)+2^r(b+1+\gamma_r)\left\| u_\rho\right\|_{\rho_\X^{te}}\Bigg).
\end{align*}
Then the desired results holds by choosing $\tilde{C}_{r,\epsilon}=\max\{C_3,C_4\}$ and the fact that  $\log\frac{4}{\delta}<\log\frac{6}{\delta}$ for $0<\delta<1.$
%\qed

\subsection{Convergence analysis of unweighted spectral algorithms under covariate shift}
In this subsection, we prove the main results for classical spectral algorithm (unweighted spectral algorithm) with covariate shift. Recall that, given two self-adjoint operators $A$ and $B$, the notation $A\succeq B$ indicates that $A-B\succeq 0$, where $A-B$ is a positive semidefinite operator. Alternatively, this condition can be expressed as $\langle Af, f\rangle_{\mathcal{H}} \geq \langle Bf, f\rangle_{\mathcal{H}}$ for all $f \in \mathcal{H}$. If $A\succeq B$, then for any operator $C$ on $\mathcal{H}$, it follows that $C^T AC\succeq C^TBC$.	
\begin{lemma}\label{lemma: operator3}
If the weight function is uniformly bounded, i.e., there exists some constant $U$ such that $|w(x)|\le U$ for all $x\in\X,$ then
	\begin{align*}
		\left\|L_K^{1/2}(\lambda I+\tilde{L}_K)^{-1/2}\right\|_{op}\le \sqrt{U}.
	\end{align*}
\end{lemma}
\begin{proof}
	For any $f\in\H_K,$
	\begin{align*}
		\langle L_Kf,f\rangle_K &=\left\langle\int_{\X} f(x)K_xd\rho_\X^{te},f\right\rangle_K=\int_{\X} (f(x))^2 w(x)d\rho_\X^{tr}\le U\int_{\X} (f(x))^2 d\rho_\X^{tr}\\&\le U\left\langle \tilde{L}_K f, f\right\rangle_K\le U(\lambda \|f\|_K^2+\left\langle \tilde{L}_K f, f\right\rangle_K)=U\left\langle \left(\lambda I+ \tilde{L}_K\right) f, f\right\rangle_K,
	\end{align*}
	which implies $
	L_K \preceq  U\left(\lambda I+ \tilde{L}_K\right).$ Then we
	$$\left(\lambda I+ \tilde{L}_K\right)^{-1/2}	L_K \left(\lambda I+ \tilde{L}_K\right)^{-1/2}\preceq  U I.$$
	Then we have
	\begin{align*}
		\left\|L_K^{1/2}\left(\lambda I+\tilde{L}_K\right)^{-1/2}\right\|_{op}^2=	\left\|	\left(\lambda I+ \tilde{L}_K\right)^{-1/2}	L_K \left(\lambda I+ \tilde{L}_K\right)^{-1/2}\right\|_{op} \leq U.
	\end{align*}
	This completes the proof.
\end{proof}
Now, we are ready to demonstrate the proof of Theorem \ref{theorem: unweighted spectral algorithm}.

\noindent{\bf Proof of Theorem \ref{theorem: unweighted spectral algorithm}.}
By the definition of $f_{\bz,\lambda}$ and the property \eqref{condition2} of the filter function $g_\lambda$, we have the following error decomposition
\begin{align*}\label{equation: error decomposition intermediate term unweighted}
	&\|f_{\bz,\lambda}-f_\rho\|_{\rho_\X^{te}}=\left\|L_K^{1/2}\left(f_{\bz,\lambda}-f_\rho\right)\right\|_K =\left\|L_K^{1/2}\left(\lambda I+ \tilde{L}_K\right)^{-1/2} \left(\lambda I+ \tilde{L}_K\right)^{1/2}  (f_{\bz,\lambda}-f_\rho)\right\|_K\\
	&=\left\|L_K^{1/2} \left(\lambda I+ \hat{L}_K\right)^{-1/2} \left(\lambda I+ \hat{L}_K\right)^{1/2} \left(\lambda I+ S_{X}^\top  S_X\right)^{-1/2} \left(\lambda I+ S_{X}^\top S_X\right)^{1/2} (f_{\bz,\lambda}-f_\rho)\right\|_K\\
	&\le \left\|L_K^{1/2} \left(\lambda I+ \tilde{L}_K  \right)^{-1/2}\right\|_{op}  \left\|\left(\lambda I+ L_K\right)^{1/2} \left(\lambda I+ S_{X}^\top  S_X\right)^{-1/2}\right\|_{op}\cdot\\
	&~~~~\Bigg[\left\|\left(\lambda I+ S_{X}^\top  S_X\right)^{1/2} \left(g_\lambda(S_{X}^\top  S_X)(S_X^\top \bar{y}-S_X^\top  S_X f_\rho)\right)\right\|_K \\
	&~~~~	+ \left\|\left(\lambda I+ S_{X}^\top  S_X\right)^{1/2} \left(g_\lambda(S_{X}^\top  S_X)S_X^\top  S_X -I\right)f_\rho\right\|_K\Bigg]\\
	%			&\le \left\|L_K^{1/2} (\lambda I+ \tilde{L}_K  )^{-1/2}\right\|_{op}  \cdot \Big[\left\|(\lambda I+ L_K)^{1/2} (\lambda I+ S_{X}^\top  S_X)^{-1/2}\right\|_{op}\cdot\\
	%			&~~~~\left\|(\lambda I+ S_{X}^\top  S_X)^{1/2}(g_\lambda(S_{X}^\top  S_X)(\lambda I+ S_{X}^\top  S_X)^{1/2} (\lambda I+ S_{X}^\top  S_X)^{-1/2}(\lambda I+\tilde{L}_K)^{\frac12} (\lambda I+\tilde{L}_K)^{-\frac12}(S_X^\top \bar{y}-S_X^\top  S_X f_\rho))\right\|_K \\
	%			&~~~~+\left\|(\lambda I+ L_K)^{1/2} (\lambda I+ S_{X}^\top  S_X)^{-1/2}\right\|_{op}\cdot \left\|(\lambda I+ S_{X}^\top  S_X)^{1/2}(g_\lambda(S_{X}^\top  S_X)S_X^\top  S_X -I)f_\rho\right\|_K\Big]\\
	&\le \left\|L_K^{1/2} \left(\lambda I+ \tilde{L}_K  \right)^{-1/2}\right\|_{op}  \cdot \\
	&~~~~\Bigg[2b\left\|\left(\lambda I+ L_K\right)^{1/2} \left(\lambda I+ S_{X}^\top  S_X\right)^{-1/2}\right\|_{op}^2\cdot\left\| \left(\lambda I+\tilde{L}_K\right)^{-\frac12}\left(S_X^\top \bar{y}-S_X^\top  S_X f_\rho\right)\right\|_K \notag\\
	&~~~~+\left\|\left(\lambda I+ L_K\right)^{1/2} \left(\lambda I+ S_{X}^\top  S_X\right)^{-1/2}\right\|_{op}\cdot \left\|\left(\lambda I+ S_{X}^\top  S_X\right)^{1/2} \left(g_\lambda(S_{X}^\top  S_X)S_X^\top  S_X -I\right)f_\rho\right\|_K\Bigg].\notag
\end{align*}
We can observe that the error decomposition above is almost the same as Proposition 2 in \citet{guolinzhou2017}, except for the additional term  $\left\|L_K^{1/2} (\lambda I+ \tilde{L}_K  )^{-1/2}\right\|_{op}$ on the right-hand side in our case.
Then by Theorem 2 in \citet{guolinzhou2017} and Lemma \ref{lemma: operator3}, with confidence at least $1-\delta,$
\begin{align*}
	\|f_{\bz,\lambda}-f_\rho\|_{\rho_\X^{te}}\le  \tilde{C} N^{-\frac{ r}{2r+\beta}}\left(\log6/\delta\right)^4,
\end{align*}
where $\tilde{C}=2\sqrt{U}C\left[4\left(\kappa^2+\kappa\sqrt{C_0} \right)^2
+ 1\right]\left(\kappa^2+\kappa \sqrt{C_0}+2\right)$. This completes the proof.
%\qed

% Acknowledgements and Disclosure of Funding should go at the end, before appendices and references

\acks{The work by J. Fan is partially supported by the Research Grants Council of Hong Kong [Project No. HKBU12302923, No. HKBU12303024], Guangdong Basic and Applied Basic Research Fund [Project No. 2024A1515011878], and Hong Kong Baptist University [Project No. RC-FNRA-IG/22-23/SCI/02]. The work by Z. C. Guo is partially supported by Zhejiang Provincial Natural Science Foundation of China [Project No. LR20A010001] and National Natural Science Foundation of China [Project  No. 12271473, No. U21A20426]. The work by L. Shi is partially supported by the National Natural Science Foundation of China (Grant No. 12171039). All authors contributed equally to this work and are listed alphabetically. The corresponding author is Zheng-Chu Guo.}


\vskip 0.2in
\bibliography{main}

\end{document}